\documentclass{article}

\PassOptionsToPackage{numbers, compress}{natbib}

\usepackage[preprint]{neurips_2025}
\usepackage[utf8]{inputenc}
\usepackage[T1]{fontenc}
\usepackage{hyperref}
\usepackage{url}
\usepackage{booktabs}
\usepackage{enumitem}
\usepackage{amsfonts}
\usepackage{nicefrac}
\usepackage{microtype}
\usepackage{xcolor}

\usepackage{amsmath}
\usepackage{amssymb}
\usepackage{mathtools}
\usepackage{amsthm}
\usepackage{bbm}
\usepackage{bm}

\usepackage{algorithm}
\usepackage{algorithmic}

\usepackage{graphicx}
\usepackage{subcaption}
\usepackage{tikz}
\usetikzlibrary{arrows.meta, positioning, calc, decorations.pathreplacing}
\usepackage{pgfplots}
\pgfplotsset{compat=1.18}

\definecolor{hmG}{HTML}{1B7340}
\definecolor{hmLG}{HTML}{5FAD65}
\definecolor{hmY}{HTML}{E8B931}
\definecolor{hmO}{HTML}{D4782F}
\definecolor{hmR}{HTML}{C0392B}

\usepackage{multirow}

\usepackage[most]{tcolorbox}

\usepackage[capitalize,noabbrev]{cleveref}

\theoremstyle{plain}
\newtheorem{theorem}{Theorem}[section]
\newtheorem{proposition}[theorem]{Proposition}
\newtheorem{lemma}[theorem]{Lemma}
\newtheorem{corollary}[theorem]{Corollary}
\theoremstyle{definition}
\newtheorem{definition}[theorem]{Definition}
\newtheorem{assumption}[theorem]{Assumption}
\newtheorem{example}[theorem]{Example}
\theoremstyle{remark}
\newtheorem{remark}[theorem]{Remark}
\theoremstyle{plain}
\newtheorem{conjecture}[theorem]{Conjecture}
\newtheorem{question}[theorem]{Question}

\newcommand{\R}{\mathbb{R}}

\newcommand{\E}{\mathbb{E}}
\newcommand{\Sym}{\mathrm{Sym}}
\newcommand{\Orth}{\mathrm{O}}

\newcommand{\Sn}{\mathbf{S}^n}
\newcommand{\Skn}{\mathbf{S}^{k,n}}

\let\hom\relax
\DeclareMathOperator{\hom}{hom}
\DeclareMathOperator{\inj}{inj}
\newcommand{\ghw}{W_{\mathrm{gh}}}

\newcommand{\ElemInv}{\Theta}
\newcommand{\SpecRelax}{\Lambda}

\newcommand{\inner}[2]{\langle #1, #2 \rangle}

\newcommand{\HG}{\mathcal{H}}

\newcommand{\EE}{\mathcal{E}}

\title{The Width Wall: A Strict Expressivity Hierarchy\\
for Hypergraph Neural Networks}

\author{
	\textbf{Fengqing Jiang}\textsuperscript{$\clubsuit$} \;\;\;  
	\textbf{Yuetai Li}\textsuperscript{$\clubsuit$} \;\;\;
	\textbf{Yichen Feng}\textsuperscript{$\clubsuit$} \;\;\;  
	\textbf{Kaiyuan Zheng}\textsuperscript{$\clubsuit$} \;\;\; \\
	\textbf{Luyao Niu}\textsuperscript{$\clubsuit$} \;\;\; 
	\textbf{Bhaskar Ramasubramanian }\textsuperscript{$\diamondsuit$}\;\;\;
	\textbf{Basel Alomair}\textsuperscript{$\spadesuit$}\textsuperscript{$\clubsuit$}\textsuperscript{$\heartsuit$}\;\;\; \\
	\textbf{Linda Bushnell}\textsuperscript{$\clubsuit$} \;\;\; 
	\textbf{Radha Poovendran}\textsuperscript{$\clubsuit$}\\
	\textsuperscript{$\clubsuit$}University of Washington \; 
	\textsuperscript{$\diamondsuit$}Western Washington University \;\\
	\textsuperscript{$\spadesuit$}King Abdulaziz City for Science and Technology \\
	\textsuperscript{$\heartsuit$}HUMAIN \\
	\texttt{rp3@uw.edu}
}

\begin{document}

\maketitle

\begin{abstract}
Hypergraphs 
{provide} 
a natural 
{framework to model} higher-order 
{interactions} 
in scientific, social, and biological systems.
Hypergraph neural networks (HGNNs) 
{aim to learn from such data, yet it remains unclear which higher-order structures these models can represent.} 
{We show that hypergraph expressivity is governed by which small patterns an architecture can detect and count}. 
We formalize this via \emph{homomorphism densities}, which measure how often a structural motif appears in 
a hypergraph.
Combining classical homomorphism-count completeness with invariant approximation, we show that homomorphism densities generate all continuous hypergraph invariants and organize them into a strict 
{hierarchy indexed by hypertree width}. 
This yields a \emph{Width Wall}: {a fundamental architectural limit beyond which no hidden dimension, training procedure or fixed-depth HGNN can represent invariants requiring wider patterns.} 
{Our framework provides} a unified characterization of 15 HGNN architectures, precisely identifies information lost by clique expansion, and motivates density-aware models that {extend expressivity beyond bounded-width message passing.} 
We experimentally validate this finding on an \textsc{Application Node Classification Suite} of real-world hypergraphs, where the Width Wall predicts when graph-reduction baselines fail and when density features help.
\end{abstract}


\section{Introduction}\label{sec:intro}

Many real-world systems- such as co-authorship networks, protein complexes, legislative coalitions- are organized by simultaneous interactions among multiple entities.
Hypergraphs provide a natural language for {modeling} these higher-order relations: 
a single hyperedge can bind an entire group, rather than forcing interaction into pairwise links~\citep{BattisonEtAl2020, BensonGleichLeskovec2016}.
This modeling advantage has driven a rapid expansion of hypergraph neural networks (HGNNs)~\citep{ChienEtAl2022, FengEtAl2019, WangEtAl2023EDHNN, YadatiEtAl2019}.

However, HGNNs are not guaranteed to learn higher-order {structures and} relations effectively from hypergraphs.
The {central} bottleneck lies in how the architecture summarizes each hyperedge. {Some methods} collapse group interactions into pairwise messages, {while others} preserve the incidence {structure, but still rely on local, symmetric aggregation over node-hypergraph edge neighborhoods}. 
For example, HGNN and HyperGCN transform each hyperedge into a graph-based propagation operator, so a $k$-way relation is ultimately processed through pairwise messages~\citep{FengEtAl2019, YadatiEtAl2019}.
{In contrast}, native models such as HNHN, UniGNN, AllSet, and ED-HNN {retain} the incidence structure, {yet stil aggregate only local symmetric summaries}
~\citep{ChienEtAl2022, DongEtAl2020HNHN, HuangYang2021, WangEtAl2023EDHNN}.
{Thus, even architectures that preserve higher-order structure may fail to capture genuinely higher-order dependencies.} 
This raises the question: \emph{when can these models actually learn genuine higher-order structure {from} the input?}

%
We address the problem through the lens of \emph{information loss}.
Consider legislative co-sponsorship: {a single} hyperedge may encode a multi-senator coalition whose political meaning depends on the {entire} group.
If an HGNN first reduces that coalition to pairwise edges, a bipartisan coalition of~$8$ {senators can} become indistinguishable from two partisan blocs of~$4$ after clique expansion, even though their downstream implications may differ sharply.
We empirically validate this failure mode on Senate co-sponsorship data~\citep{Fowler2006} using: (1) HyperGCN, which projects hyperedges to pairwise edges and reaches $55\%$ node-classification accuracy, and (2) {a model that aggregates} over full hyperedge memberships and reaches $93\%$.
This gap validates \emph{architectural information loss}: once the relevant higher-order pattern is discarded, {no increase in hidden dimension or training can} 
recover it.

In this paper, we characterize information loss via a notion named \textbf{hypergraph expressivity}. 
Our key insight is that hypergraph expressivity is governed by which small higher-order patterns an architecture can detect and count.
We formalize these counts {using} \emph{homomorphism densities}, which measure how often a pattern appears in the input hypergraph.
Using classical homomorphism-count completeness, we show that these densities separate non-isomorphic hypergraphs and generate all continuous hypergraph invariants
- {that is, all }continuous permutation-invariant functions on the compact hypergraph-tensor space.
{Consequently}, any continuous structural property of a fixed-size hypergraph can be approximated from pattern densities. 
{This perspective yields a diagnostic basis that makes explicit which information is lost when a model} replaces a hypergraph. 

The density basis {induces} 
a graded expressivity scale for HGNNs.
Let $\mathcal{H}_r$ denote the {class of} continuous invariants generated by homomorphism densities of patterns with generalized hypertree width at most~$r$, and let $\mathcal{I}_{\mathrm{all}}$ denote all continuous hypergraph invariants.
We show that these classes form a strict hierarchy, $\mathcal{H}_1 \subsetneq\mathcal{H}_2 \subsetneq \cdots \;\subsetneq\; \mathcal{I}_{\mathrm{all}}$.
An HGNN whose aggregation scheme reaches level~$r$ can compute invariants generated by patterns of width $\leq r$, but cannot {represent invariants requiring larger width}. 
We term this obstruction the \emph{Width Wall}: an architectural ceiling determined by pattern densities accessible to the model, rather than by hidden dimension or optimization. 
{This correspondence gives a principled way to classify existing HGNN architectures within the hierarchy}. 

\section{Homomorphism Features and the Expressivity Framework}\label{sec:invariants}

\paragraph{Notation and Background.}
We write $[n]=\{1,\ldots,n\}$. 
A $k$-uniform hypergraph is
$\HG=([n],\EE)$, where $\EE\subseteq \binom{[n]}{k}$ is a set of unordered $k$-node hyperedges. 
The adjacency tensor of $\HG$ is the order-$k$ array $\bm A$ with $\bm A_{i_1,\ldots,i_k}=1$ if $(i_1,\ldots,i_k)\in\EE$, and $0$ otherwise. 
Because hyperedges are unordered, $\bm A$ is symmetric, i.e., {invariant under permutations of its} $k$ indices. 
Let $\Skn$ denote the space of order-$k$, dimension-$n$ symmetric tensors. 
Thus, $\bm A\in\Skn$ for any $k$-uniform hypergraph. 
The case $k=2$ recovers ordinary graphs, where $\mathbf{S}^{2,n}$ is the space of symmetric $n\times n$ adjacency matrices. 
Weighted hypergraphs allow entries in $[0,1]$, yielding a compact domain $[0,1]^{\Skn}$.
For non-uniform hypergraphs with maximum edge size $K$, we {represent adjacency as}
$\bm A=(A^{(2)},\ldots,A^{(K)})$, where $A^{(j)}\in\mathbf{S}^{j,n}$ {records} hyperedges of size $j$.
Finally, $\Sym(n)$ denotes the group of vertex permutations, acting on $\bm A$ by simultaneously relabeling all tensor indices. 
A hypergraph
invariant is a function unchanged under this action.

\subsection{Tensor Domain and Hypergraph Invariants}\label{sec:tensor-invariants}

\paragraph{Overview.}
In graph theory, homomorphism counts {provide a canonical coordinate system for finite graphs}. They  
characterize graphs up to isomorphism, underlie dense graph limits, and connect to logical and WL-style expressivity~\citep{DellGroheRattan2018,Lovasz2012,LovaszSzegedy2006}.
We extend this viewpoint to hypergraphs and use it to bridge combinatorial patterns and neural architectures.
In this section, we characterize the expressivity of an HGNN {through} the class of \emph{hypergraph homomorphism densities} it can access, which {measure} how often a structural pattern appears in the {input} 
hypergraph~\citep{ElekSzegedy2012,Lovasz1967,Lovasz2012,Zhao2015}.
We then treat these densities as \emph{invariant coordinates}: each pattern defines one label-independent coordinate of the hypergraph, and different architecture classes expose different subsets of these coordinates.
Further algebraic structure appears in Appendix~\ref{sec:proof-algebra}, with an independent invariant-theoretic derivation in Appendix~\ref{app:algebraic}; concrete examples of the resulting hypergraph invariants are in Appendix~\ref{app:examples}.

Our discussion proceeds in three steps.
First, we represent hypergraphs as symmetric tensors, {making invariance to vertex ordering within each hyperedge}.  
Second, we introduce pattern densities and their aligned-template counterparts.
Finally, we show that pattern densities generate all continuous invariant functions.
{Under this view}, expressivity reduces to asking which coordinates an architecture can actually compute.
If an architecture cannot access patterns generating a target invariant,
increasing hidden dimension alone will not recover that invariant.

\paragraph{Tensor Domain and Invariance.}
To analyze what a model can learn from hypergraph structure, we separate structural information from arbitrary vertex labels.
We use the adjacency tensor representation 
as the common input space for hypergraph invariants, which later serves as features for analyzing HGNN expressivity.
A hypergraph invariant must be unchanged under vertex relabeling.
For a permutation $\Pi\in\Sym(n)$ of the vertex set $[n]$, the relabeling action is
\[
  (\Pi \cdot \bm{A})_{i_1,\ldots,i_k}
  = \bm{A}_{\Pi^{-1}(i_1),\ldots,\Pi^{-1}(i_k)} .
\]
Together, tensor symmetry and vertex-relabeling action 
{define the}
domain of continuous hypergraph invariants.
Unweighted hypergraphs {correspond to} $\{0,1\}$ vertices of the compact cube $[0,1]^{\Skn}$, while weighted hypergraphs fill the cube.
{Accordingly, }
a continuous hypergraph invariant is a continuous function on this tensor domain that is constant on $\Sym(n)$-orbits.
For non-uniform hypergraphs with maximum edge size $K$, the same {construction} applies component-wise to the tuple $(A^{(2)},\ldots,A^{(K)})$, where $A^{(j)}\in\mathbf{S}^{j,n}$ records hyperedges of size $j$.
We now formalize hypergraph invariants as label-independent structural functions.


\begin{definition}[Hypergraph invariant]\label{def:invariant}
A function $f: \Skn \to \R$ is a \emph{hypergraph invariant} if $f(\Pi \cdot \bm{A}) = f(\bm{A})$ for all $\bm{A} \in \Skn$ and all $\Pi \in \Sym(n)$.
\end{definition}

\subsection{Homomorphism Densities and Pattern Alignment}\label{sec:hom-density-alignment}

We use two complementary families of invariants.
The first is the family of \emph{homomorphism densities}, which count how often a finite pattern appears in the target {hypergraph} and provide the count-based algebra {underlying our} completeness and density results.
The second is the family of \emph{pattern alignment scores}, which compare a target hypergraph to a template under {optimal} 
vertex relabeling and provide optimization-based template invariants.

Let $\mathcal{F}_k$ denote the set of all finite $k$-uniform hypergraph patterns.
Let $F$ be {such a pattern} 
and $H$ be a target $k$-uniform hypergraph. 
A map $\phi:V(F)\to V(H)$ places vertices of the pattern into the target, where $V(\cdot)$ denotes the vertex set of a hypergraph.

\begin{definition}[Hypergraph homomorphism {\citep{Lovasz1967,Lovasz2012}}]\label{def:hom}
Let $F = ([p], E_F)$ and $H = ([n], E_H)$ be $k$-uniform hypergraphs.
A \emph{homomorphism} $\phi: V(F) \to V(H)$ is a map such that hyperedge $(\phi(i_1), \ldots, \phi(i_k)) \in E_H$ whenever hyperedge $(i_1, \ldots, i_k) \in E_F$.
The \emph{homomorphism count} is
\[
  \hom(F, H) \;=\; \bigl|\{\phi: \phi \text{ is a homomorphism from } [p] \text{ to } [n]\}\bigr|.
\]
The \emph{homomorphism density} is the fraction of all maps from $[p]$ to $[n]$ that are homomorphisms:
\[
  t(F, H) \;=\; \frac{\hom(F, H)}{n^p}.
\]
\end{definition}

For weighted target hypergraphs with adjacency tensor $\bm{A} \in [0,1]^{\Skn}$ and a pattern $F=([p],E_F)$, the homomorphism density becomes a multilinear polynomial~\citep{ElekSzegedy2012,LovaszSzegedy2006,Zhao2015}:
\begin{equation}\label{eq:hom-density}
  t(F, \bm{A}) \;=\; \frac{1}{n^p} \sum_{\phi: [p] \to [n]} \;\prod_{e \in E_F} \bm{A}_{\phi(e)},
\end{equation}
where $\bm{A}_{\phi(e)} = \bm{A}_{\phi(i_1), \ldots, \phi(i_k)}$ for $e = (i_1, \ldots, i_k)$.
Thus, $t(F,\bm A)$ is a polynomial of degree $|E_F|$ in the tensor entries.
These features are continuous and permutation-invariant, and their pattern {encodes} 
the higher-order interaction being measured.
For example, a single edge measures edge density, a star measures local overlap, and more entangled patterns capture higher-order coordination.
Thus, the family $\{t(F,\cdot):F\in\mathcal F_k\}$ should be viewed not merely as
a collection of statistics, but as a coordinate map {that records all pattern measurements of the hypergraph}. 

Homomorphism densities form an algebra under pointwise addition and multiplication.{In particular}, 
multiplying two densities is equivalent to counting the disjoint union of the two query patterns $t(F_1, \bm{A}) \cdot t(F_2, \bm{A}) = t(F_1 \sqcup F_2, \bm{A})$~\citep[Ch.\,5]{Lovasz2012}.
We write $\mathfrak{A}_k$ for the unital algebra generated by $\{t(F,\cdot): F \in \mathcal{F}_k\}$, with constants given by the empty pattern $t(\varnothing,\bm A)=1$. 
{Further details on algebraic structure appear in }Appendix~\ref{sec:proof-algebra}.

For the architectural constructions later in the paper, {we also introduce pattern alignment as a template-based} 
analogue of pattern counting.
Rather than averaging 
over all vertex maps, this feature compares a template tensor with the target under the best vertex relabeling.
It is therefore invariant to vertex labels and provides an optimization-based invariant.

\begin{definition}[Pattern alignment score]\label{def:pattern-alignment}
For any template tensor $P \in \Skn$, the \emph{pattern alignment score} is
\begin{equation}\label{eq:elementary}
  \ElemInv_P(\bm{A}) \;=\; \max_{\Pi \in \Sym(n)} \inner{P}{\Pi \cdot \bm{A}}.
\end{equation}
\end{definition}

The score $\ElemInv_P$ is invariant under vertex relabeling of $\bm A$ by construction.
Homomorphism densities provide the count-based algebra used in the completeness and density proofs, while alignment scores provide the optimization-based template features.

\subsection{Completeness and Density}\label{sec:completeness}

The preceding constructions provide concrete invariant coordinates. We now ask
whether they {capture} all structural information available to a continuous
invariant learner.
\paragraph{Finite densities model structural information.}
We prove that finite collections of pattern densities can approximate arbitrary continuous invariant targets on the fixed-$n$ tensor domain.

\begin{theorem}[Pattern-density universality with fixed $n$]\label{thm:density}
Fix $n \geq 1$.
Let $\mathcal{K} = [0,1]^{\Skn}$ be the space of weighted $k$-uniform hypergraphs on $[n]$, and let $C(\mathcal{K} / \Sym(n))$ denote the Banach space of continuous $\Sym(n)$-invariant functions on $\mathcal{K}$ with the supremum norm.
Then for every continuous hypergraph invariant $f$ and every $\epsilon > 0$, there exist patterns $F_1, \ldots, F_m$ and a polynomial $q: \R^m \to \R$ such that
\[
  \sup_{\bm{A} \in \mathcal{K}} \bigl| f(\bm{A}) - q\bigl(t(F_1, \bm{A}), \ldots, t(F_m, \bm{A})\bigr) \bigr| < \epsilon.
\]
\end{theorem}

Theorem \ref{thm:density} shows that any fixed-size structural property that is continuous and {invariant to vertex relabeling} 
(e.g., a spectral, cut, or topological statistic) can be approximated by a polynomial in finitely many pattern densities.
In this sense, pattern densities form a complete coordinate system for
continuous invariants on the fixed tensor domain: {once all pattern coordinates are available, no additional continuous label-independent information is required}. 
{The fixed $n$ setting matches the neural-network regime, which }
{operates on a tensor of fixed size.} 
Extensions to non-uniform hypergraphs are given in Appendix~\ref{app:nonuniform}; extensions to hypergraphon limits are in~\citep{ElekSzegedy2012, Zhao2015}.

\section{Hypergraph Invariants Induce a Strict Hierarchy}\label{sec:expressivity}

{We have shown that each pattern density $t(F,\bm A)$ defines an invariant coordinate and that}
these coordinates are complete for
continuous invariants on a fixed tensor domain.  
{We now investigate which of these coordinates are accessible to different HGNN architectures}. 

\subsection{Architecture Scope}

Before comparing architectures, we fix the scope of our claims.
Throughout, 
we consider label-free local hypergraph neural architectures
satisfying the following {conditions}.

\begin{assumption}[Architecture scope]\label{assump:arch-scope}
The architecture satisfies:
\textbf{(C1)}~permutation equivariance, with no node identifiers or order-dependent operations;
\textbf{(C2)}~locality, with aggregation only from incident hyperedge neighborhoods;
\textbf{(C3)}~symmetric aggregation, such as sum, mean, max, DeepSets, or attention over a multiset;
\textbf{(C4)}~no symmetry-breaking features, such as Laplacian eigenvectors, random IDs, or global positional encodings; and
\textbf{(C5)}~deterministic structure-dependent preprocessing.
\end{assumption}

\textbf{Key observation.}
HGNN architectures can be compared by the pattern densities they can
represent. Clique-expansion methods, such as HGNN~\citep{FengEtAl2019} and
HyperGCN~\citep{YadatiEtAl2019}, factor through the pairwise projection
$\phi:\Skn\to\Sn$, {retaining only graph-level structure}. 
Native hyperedge message-passing methods, such as HNHN~\citep{DongEtAl2020HNHN},
UniGNN~\citep{HuangYang2021}, AllSet~\citep{ChienEtAl2022}, and
ED-HNN~\citep{WangEtAl2023EDHNN}, {preserve incidence but access only patterns that can be assembled through local aggregation.}
{In contrast, the full} pattern-density algebra from
Theorem~\ref{thm:density} {imposes no such} restriction. 
The hierarchy below formalizes these successive enlargements as nested invariant classes.

For an HGNN architecture class $\mathcal{F}$, $\mathcal{I}(\mathcal{F})$ is the set of continuous hypergraph invariants that functions in $\mathcal{F}$ can approximate. 
Thus $\mathcal{I}(\mathcal{F})$ records the invariant information representable by the architecture.

We first formalize the loss induced by the clique-expansion methods.

\begin{proposition}[Clique-expansion limitation]\label{prop:clique-limitation}
Let $\phi: \Skn \to \Sn$ be the clique-expansion map.
Suppose a clique-expansion HGNN has the form
$f_\theta(\bm{A}) = g_\theta(\phi(\bm{A}))$, where
$g_\theta:\Sn \to \R^d$ is a graph neural network applied to the
clique-expanded graph.
If two hypergraphs have the same clique expansion,
$\phi(\bm{A}_1)=\phi(\bm{A}_2)$, then $f_\theta(\bm{A}_1)=f_\theta(\bm{A}_2)$
for every parameter choice $\theta$.
\end{proposition}

Proposition~\ref{prop:clique-limitation} shows clique-expansion methods discard all information not determined by the pairwise shadow. 
{However, native incidence-based HGNNs remain limited to patterns accessible through local aggregation}. 

\subsection{Width, Invariance Classes, and the Width Wall}

{We investigate how hyperedge structure must be coordinated jointly in order to count a pattern. 
We measure this coordination complexity using \emph{generalized hypertree width} \citep{GottlobLeoneScarcello2002}}. 

\begin{definition}[Generalized hypertree width \citep{GottlobLeoneScarcello2002}]\label{def:gh-width}
A \emph{tree decomposition} of a hypergraph $F$ is a tree $T$ whose nodes $t$ are assigned vertex sets $B_t \subseteq V(F)$, called \emph{bags}, covering all vertices and hyperedges, such that for each $v \in V(F)$ the bags containing $v$ form a connected subtree.
The \emph{generalized hypertree width} $\ghw(F)$ is the minimum over all tree decompositions of $F$ of the maximum number of hyperedges of $F$ needed to cover any single bag.
\end{definition}
Intuitively, $\ghw(F)$ measures how {many hyperedges}
must be tracked jointly when counting $F$: 
tree-like patterns have $\ghw = 1$, while dense or highly connected patterns require larger width.

Generalized hypertree width 
{induces a graded structure on the density coordinates: each pattern density }
$t(F,\cdot)$ has a width, and increasing the allowed width exposes strictly richer {invariant information}. 
The architectural question {then} becomes which portion of this graded density algebra a given model can access.  
Clique-expansion models access only graph-level coordinates of the pairwise shadow, native message-passing models access a bounded-width portion of the hypergraph coordinates, and the full density algebra imposes no width restriction.
We formalize these three regimes as invariant function classes.

\begin{definition}[Invariant classes via pattern complexity]\label{def:invariant-classes}
For a width budget $R < \infty$, define the following classes of continuous
hypergraph invariants:
\begin{enumerate}
  \item $\mathcal{I}_{\mathrm{CE}}$: invariants computable from pattern densities of \emph{graph patterns} on the clique expansion (i.e., functions factoring through $\phi: \Skn \to \Sn$).
  \item $\mathcal{I}_{\mathrm{NH}}^{(R)}$: invariants generated by
  $\mathcal{I}_{\mathrm{CE}}$ together with hypergraph pattern densities
  $t(F,\cdot)$ for patterns satisfying $\ghw(F) \leq R$.
  \item $\mathcal{I}_{\mathrm{all}}$: all continuous hypergraph invariants on $\Skn$ (generated by the full pattern density algebra $\mathfrak{A}_k$, by Theorem~\ref{thm:density}).
\end{enumerate}
\end{definition}

For a positive integer $r$, the \emph{width-$r$ level} consists of pattern
densities $t(F,\cdot)$ with $\ghw(F) \leq r$.
Equivalently, level $r$ contains invariants generated by hypergraph patterns
whose counting can be organized using at most $r$ covering hyperedges in any
bag of a generalized hypertree decomposition.

\begin{definition}[Width Wall]\label{def:width-wall}
For an architecture class whose invariant information is contained in
$\mathcal{H}_r$, the \emph{Width Wall at level $r$} is the obstruction formed
by invariants in $\mathcal{I}_{\mathrm{all}} \setminus \mathcal{H}_r$.
Equivalently, it is the boundary between the pattern densities the architecture
can access, those with $\ghw(F)\leq r$, and the wider pattern densities needed
to separate some hypergraphs.
\end{definition}

{Thus, the Width Wall is a structural barrier. Once an architecture exposes only width-$r$ coordinates, no increase in hidden dimension, readout complexity, or training can recover invariants 
with $\ghw > r$.}


\subsection{Expressivity Hierarchy and Architectural Consequences}

Invariant classes translate architectural mechanisms into invariant function classes.
$\mathcal{I}_{\mathrm{CE}}$ contains exactly the invariants that depend only on
the pairwise shadow.  $\mathcal{I}_{\mathrm{NH}}^{(R)}$ is a native-incidence
reference class: it retains the CE information and adds hypergraph motifs whose
joint coordination complexity is at most $R$.  The parameter $R$ should be read
as an architectural width budget, determined by the message-passing mechanism
and depth through Proposition~\ref{prop:arch-width}.  Finally,
$\mathcal{I}_{\mathrm{all}}$ is the ideal reference class from
Theorem~\ref{thm:density}, with no pattern-complexity restriction imposed.

\begin{theorem}[Strict expressivity hierarchy]\label{thm:hierarchy}
Fix a native-incidence reference class with width budget $R \geq 3$.
For $k$-uniform hypergraphs with $k \geq 3$ on
$n \geq \max\{13, N_{\mathrm{CFI}}^{(k)}(R)\}$ vertices, we have $\mathcal{I}_{\mathrm{CE}}
  \subsetneq
  \mathcal{I}_{\mathrm{NH}}^{(R)}
  \subsetneq
  \mathcal{I}_{\mathrm{all}},$
where $N_{\mathrm{CFI}}^{(k)}(R)$ denotes the number of vertices in the
smallest $k$-uniform Cai--Furer--Immerman (CFI) lift whose base graph has girth
greater than $2R+1$.
\end{theorem}

The proof and explicit separation witnesses are given in Appendices~\ref{app:proofs} and~\ref{app:separation}.
Theorem~\ref{thm:hierarchy} says that each relaxation of the information
bottleneck strictly adds invariant information.  Moving from clique expansion
to native incidence features recovers higher-order structure lost by the
pairwise projection.  Moving from any fixed width budget $R$ to the full
pattern-density algebra recovers invariants whose witnessing patterns require
larger structural complexity.


\begin{theorem}[Infinite expressivity hierarchy]\label{thm:infinite-hierarchy}
For each $r \geq 1$, define $\mathcal{H}_r$ as the class of continuous invariants approximable using pattern densities $t(F, \cdot)$ where $F$ has generalized hypertree width $\ghw(F) \leq r$.
Then for every $r \geq 1$, $k \geq 3$, and $n \geq N_{\mathrm{CFI}}^{(k)}(r)$,
the level-$r$ inclusion is strict, i.e., $\mathcal{H}_r \;\subsetneq\; \mathcal{H}_{r+1}.$
Each strict inclusion is witnessed by explicit families of non-isomorphic hypergraph pairs that are $\mathcal{H}_r$-indistinguishable but $\mathcal{H}_{r+1}$-distinguishable.
\end{theorem}


Theorem \ref{thm:infinite-hierarchy} shows that the Width Wall is strict at every level~$r$: the CFI construction provides explicit invariants that 
{cannot be approximated by any architecture restricted to patterns with $\ghw(F) \leq r$, }
regardless of hidden dimension or training procedure (Corollary~\ref{cor:impossibility} in Appendix~\ref{app:proof-hierarchy}). 
{Accessing these invariants requires increasing the available hypertree width.}
Appendix~\ref{app:architecture} summarizes how modern HGNN architectures map into this hierarchy.
Appendix~\ref{app:structural-frontiers} studies related structural frontiers for Sinkhorn relaxations and spectral-width thresholds, with the Steiner-pair verification in Appendix~\ref{app:steiner}.

\begin{proposition}[Architecture--pattern width correspondence]\label{prop:arch-width}
Suppose Assumption~\ref{assump:arch-scope} holds.
A message-passing HGNN with $L$ layers and aggregate-combine operations over hyperedge neighborhoods is bounded by homomorphism information from patterns of generalized hypertree width at most $c_k L$, where $c_k \geq 1$ is a constant depending on the arity $k$ and the aggregation scheme (for standard $k$-ary message passing, $c_k = 1$~\citep{ScheidtSchweikardt2023}).
In particular, it cannot distinguish hypergraphs that agree on all such counts.
Conversely, sufficiently expressive aggregate-combine maps can realize the corresponding bounded-width homomorphism-count features.
\end{proposition}

An $L$-layer HGNN uses at most $c_k L + 1$ variables to describe its computation tree, and the resulting logic $C^{c_k L + 1}$ captures exactly the homomorphism counts of patterns with $\ghw(F) \leq c_k L$.
Crucially, $L$ is a fixed architectural parameter that does not grow with $n$, so the set of accessible patterns has bounded hypertree width regardless of the input size.
In the terminology of Definition~\ref{def:width-wall}, such a model faces the Width Wall at level $c_k L$.
We use the known WL--homomorphism correspondence: two $k$-uniform hypergraphs
are $r$-GWL equivalent if and only if they agree on all homomorphism counts
$\hom(F, \cdot)$ from patterns $F$ with $\ghw(F) \leq r$
\citep{Scheidt2024, ScheidtSchweikardt2023}.

\paragraph{Quantitative Learnability.}
The density theorem shows that pattern densities \emph{can} approximate any
continuous invariant, but it does not specify how many samples,
features, or parameters are needed.  Appendix~\ref{app:quantitative} collects
three quantitative complements that clarify this point.
\textbf{(a)~Concentration.}
For any fixed pattern $F$, its density $t(F,H)$ can be estimated by sampling
random vertex maps.  Additive accuracy~$\epsilon$ with failure probability
$\delta$ requires $O(\log(1/\delta)/\epsilon^2)$ samples, independent of the
number of vertices~$n$ (Proposition~\ref{prop:pattern-concentration-full}).
\textbf{(b)~Dimension counting.}
The Molien series counts the invariant polynomial degrees of freedom.  When we
restrict to patterns with $\ghw(F) \leq r$, the available polynomial invariant
space is strictly smaller; this gives an algebraic measure of the approximation
cost imposed by a bounded-width architecture (Proposition~\ref{prop:approx-dim-full}).
\textbf{(c)~Generalization.}
For the template-alignment architecture introduced later in Section \ref{sec:invnet}, the Rademacher bound scales as
$O(B_P B_A B_W^D \sqrt{D \log m / N})$.  Here $N$ is the number of training
examples, $D$ and $m$ are the depth and width of the readout MLP, and $B_P$,
$B_A$, and $B_W$ bound the template, input, and readout weights respectively.
The number of learned templates $J$ enters only logarithmically
(Proposition~\ref{prop:rademacher-full}).

\section{InvNet: An Invariant-Theoretic Architecture}\label{sec:invnet}

We propose \textsc{InvNet}, an HGNN designed to extend beyond local message passing.
\textsc{InvNet} augments a standard local HGNN with a global branch that learns permutation-invariant pattern-alignment features. 
{The model thus combines two complementary sources of information: local incidence-based messages and global template scores that compare the input hypergraph to higher-order structural patterns. We emphasize that \textsc{InvNet} serves primarily as a proof-of-concept architecture illustrating how invariant features can extend expressivity beyond bounded-width message passing.}

Concretely, \textsc{InvNet} computes
\begin{equation}\label{eq:invnet}
  f_\theta(\bm{A}) \;=\; \mathrm{MLP}_\theta\!\Bigl( \underbrace{\mathrm{ReadOut}\bigl(\bm{h}^{(L)}\bigr)}_{\text{local branch}} \;\Big\|\; \underbrace{\bigl[\hat{\ElemInv}_{P_1}(\bm{A}) - \alpha_1,\; \ldots,\; \hat{\ElemInv}_{P_J}(\bm{A}) - \alpha_J\bigr]}_{\text{invariant branch}} \Bigr),
\end{equation}
where $\bm{h}^{(L)}$ are node embeddings {obtained} from $L$ rounds of hypergraph message passing.
The templates $P_1,\ldots,P_J$ are learned global patterns, and each feature $\hat{\ElemInv}_{P_j}(\bm{A})$ approximates the {optimal} alignment between $P_j$ and the input hypergraph over all vertex relabelings. 
{Since exact optimization over $\Sym(n)$ is a combinatorial assignment problem, we relax permutation matrices to the Birkhoff polytope $\mathcal{B}_n$ and compute approximate solutions via }
Sinkhorn iterations. 
See Appendix~\ref{app:invnet-details} for details.

This construction illustrates how the invariant hierarchy can guide model design.
The local branch 
{retains expressive }
power of native message-passing HGNNs, while the invariant branch introduces global template features that are not contrained by the bounded-width computation tree of a fixed-depth HGNN.
In terms of Definition~\ref{def:width-wall}, the invariant branch is designed to introduce coordinates beyond the local branch's Width Wall rather than merely increasing {capacity within }
the same bounded-width class.

\begin{theorem}[InvNet expressivity]\label{thm:invnet-expressivity}
\textsc{InvNet} strictly extends bounded-depth message passing $\mathcal{I}_{\mathrm{NH}} \subsetneq \hat{\mathcal{I}}_{\mathrm{InvNet}} \subseteq \mathcal{I}_{\mathrm{all}}$ while remaining within continuous invariants.
\end{theorem}

\section{Experiments}\label{sec:experiments}

Our experiments {evaluate} whether the invariant hierarchy predicts empirical failure modes of HGNN architectures on real-world hypergraph learning tasks.
{Specifically}, we test three claims suggested by the 
{theory}: 
(i)~clique-expansion (CE) models can lose higher-order information; (ii)~native message-passing/set-function HGNNs can hit the Width Wall (Definition~\ref{def:width-wall}) when the useful invariant lies outside the bounded-width summaries exposed by local aggregation; and (iii)~density-aware features {improve performance}
when the backbone leaves structural headroom. 
{These experiments directly probe whether architectural performance aligns with invariant classes predicted by the hierarchy.}

\subsection{Experimental Setup}\label{sec:exp-setup}

\paragraph{Models.}
We compare three architecture classes.
The CE tier contains HGNN~\citep{FengEtAl2019} and HyperGCN~\citep{YadatiEtAl2019}, which reduce hyperedges to pairwise graph structure.
The native tier {includes} AllDeepSets, AllSetTransformer~\citep{ChienEtAl2022}, HNHN~\citep{DongEtAl2020HNHN}, and UniGNN~\citep{HuangYang2021}, which operate directly on incidence structure via bounded-depth message passing or set aggregation.
The density-aware tier contains our \textsc{DensNet-D} (\textsc{PDN} features fused with an AllDeepSets backbone; Appendices~\ref{app:invnet-details}--\ref{app:densnet}).
\textsc{DensNet-D} is a scalable density-aware HGNN derived from the invariant-feature principle behind \textsc{InvNet}. \textsc{DensNet-D} estimates per-node pattern densities by Monte Carlo sampling over each node's star neighborhood and fuses them with an AllDeepSets backbone through a learned gate. 
This design preserves the intended role of invariant structural coordinates while avoiding the quadratic Sinkhorn bottleneck on large hypergraphs.
For ANCS node classification we additionally report MLP, ED-HNN~\citep{WangEtAl2023EDHNN}, and SheafHyperGNN~\citep{DutaEtAl2023Sheaf} where applicable.

\paragraph{Task and Datasets.}
The \textsc{Application Node Classification Suite} (ANCS) 
{evaluates}
transductive node classification on empirical hypergraphs.
Nodes carry {benchmark-provided features}, 
hyperedges encode observed higher-order relations such as co-citation groups, 
legislative coalitions, or gene--disease associations, and the model predicts node labels under the standard random-split protocol described in Appendix~\ref{app:experiments}. 
%
ANCS includes 7 empirical hypergraph node-classification benchmarks, grouped by mean hyperedge size (additional detail in Appendices~\ref{app:datasets} and~\ref{app:dataset-stats}):
\begin{itemize}
    \item \textbf{Cora} (\emph{pairwise}, $\bar{k}=3.03$): a co-citation hypergraph with papers as nodes and hyperedges induced by common citing papers; node labels are paper topics and features are bag-of-words vectors~\citep{ChienEtAl2022, YadatiEtAl2019}.
    \item \textbf{Citeseer} (\emph{pairwise}, $\bar{k}=3.20$): a co-citation hypergraph constructed analogously to Cora, with document-topic labels and sparse text features~\citep{ChienEtAl2022, YadatiEtAl2019}.
    \item \textbf{PubMed} (\emph{mixed}, $\bar{k}=4.35$): a larger co-citation hypergraph over biomedical papers, using the benchmark-provided paper features and disease-category labels~\citep{ChienEtAl2022, YadatiEtAl2019}.
    \item \textbf{Cora-CA} (\emph{mixed}, $\bar{k}=4.28$): a co-authorship version of Cora in which each hyperedge groups the co-authors of one publication, with the same paper-topic prediction target~\citep{ChienEtAl2022}.
    \item \textbf{House} (\emph{higher-order}, $\bar{k}=34.89$): a legislative committee hypergraph from the U.S. House, where hyperedges are committee memberships and task is binary party classification~\citep{ChienEtAl2022,Fowler2006}.
    \item \textbf{Senate Bills} (\emph{higher-order}, $\bar{k}=7.96$): a dense legislative co-sponsorship hypergraph in which each bill forms a hyperedge over sponsoring senators, again evaluated with binary party labels~\citep{ChienEtAl2022,Fowler2006}.
    \item \textbf{Gene-Disease} (\emph{higher-order}, $\bar{k}=14.0$): a biomedical association hypergraph whose hyperedges connect genes to disease contexts, with a 21-class node-classification target under the benchmark protocol~\citep{ChienEtAl2022}.
\end{itemize}

\paragraph{Metrics.}
All tables report test classification accuracy in percent.
For ANCS benchmarks, we report mean accuracy and standard deviations over $5$--$10$ seeds.
Bold indicates the best result in each column and underlining denotes the second best.

\paragraph{Parameter setup.}
All ANCS methods use Adam with learning rate $10^{-3}$, weight decay $0$, gradient clipping $1.0$, label smoothing $0.1$, cosine annealing, and $200$ epochs ($400$ for Senate Bills).
Baselines use hidden dimension $64$ and dropout $0.2$.
\textsc{DensNet-D} uses hidden dimension $128$, dropout $0.5$, $200$ Monte Carlo density samples, and a profiler-set gate: $b=-5.0$ for pairwise datasets, $b=-2.0$ for mixed datasets, and $b=0.0$ for higher-order datasets.
The baseline settings follow the AllSet benchmark protocol~\citep{ChienEtAl2022}; \textsc{DensNet-D}'s capacity and sampling parameters are chosen by validation grid search, while its gate and freezing schedule are determined by the \textsc{DatasetProfiler} from dataset statistics.
Full training, hyperparameter-selection, and compute-budget details are in Appendices~\ref{app:hyperparams} and~\ref{app:compute}.

\subsection{Experimental Results}\label{sec:exp-results}

\paragraph{Finding 1: architecture class matters most when hyperedges are genuinely higher order.}
{The ANCS results show that performance is primarily determined by architecture class.} 
We observe that models that retain incidence-level or density-aware summaries {perform best when the dataset contains label-relevant higher-order structure.} 
{This trend appears across most datasets.} 
CE models {consistently underperform} the best native or density-aware method on Cora, Citeseer, PubMed, Cora-CA, House, Senate, and Gene-Disease. The largest gap occurs on Senate Bills, where \textsc{DensNet-D} and AllDeepSets both reach $92.7\%$ while HyperGCN reaches only $55.4\%$.
Here, reducing each bill coalition to pairwise edges appears to discard signals that native set-function models can preserve. 
{Importantly, performance does not follow a simple monotonic trend in mean hyperedge size. }
HGNN and HyperGCN do not degrade uniformly 
from pairwise to mixed to higher-order regimes.
Instead, the relevant comparison is within each dataset, where architecture choices determine which structural summaries are available.
Appendix~\ref{app:negative} analyzes the main anomaly and ceiling cases, including Senate Bills and Gene-Disease.
Thus Table~\ref{tab:node-class} supports the overall architectural trend while keeping the claim model-specific: pairwise reductions are often fragile when large empirical hyperedges carry label signal, while density-aware features help primarily when the backbone leaves structural headroom.

\begin{table}[t]
\caption{Node classification accuracy (\%), grouped by architecture tier. Best \textbf{bold}, second \underline{underlined}. MLP is a feature-only reference. $^\ddagger$Published results use architecture-specific recipes and exclude Senate/House; all other methods are re-run with the shared ANCS recipe. Baselines use hidden${}=64$; \textsc{DensNet-D} uses hidden${}=128$.}
\label{tab:node-class}
\vskip 0.05in
\centering
\small
\setlength{\tabcolsep}{2.5pt}
\begin{tabular}{@{}l cc cc ccc @{}}
\toprule
& \multicolumn{2}{c}{\textit{Pairwise ($\bar{k} {<} 4$)}} & \multicolumn{2}{c}{\textit{Mixed ($4 {\leq} \bar{k} {<} 7$)}} & \multicolumn{3}{c}{\textit{Higher-order ($\bar{k} {\geq} 7$)}} \\
\cmidrule(lr){2-3} \cmidrule(lr){4-5} \cmidrule(lr){6-8}
\textbf{Method} & \textbf{Cora} & \textbf{Citeseer} & \textbf{PubMed} & \textbf{Cora-CA} & \textbf{House} & \textbf{Senate} & \textbf{Gene-Dis.} \\
\midrule
MLP & 74.1 & 72.3 & 87.1 & 74.6 & 70.4 & 61.2 & 36.4 \\
\midrule
\multicolumn{8}{@{}l}{\textit{CE Tier (clique expansion)}} \\
HGNN & 69.9 & 68.1 & 71.2 & 76.8 & 62.8 & \underline{87.6} & 85.7 \\
HyperGCN & 71.0 & 68.1 & 78.3 & 74.3 & 63.9 & 55.4 & 81.3 \\
\midrule
\multicolumn{8}{@{}l}{\textit{Native Tier (set-function / message passing)}} \\
HNHN & 72.8 & 68.3 & 70.0 & 76.6 & 68.5 & 79.6 & \textbf{87.8} \\
UniGNN & 77.3 & 72.4 & 87.6 & 82.9 & \underline{70.9} & 78.8 & 85.4 \\
AllDeepSets & 76.2 & 70.9 & 87.2 & 80.6 &1 67.8 & \textbf{92.7} & \underline{86.1} \\
ED-HNN$^\ddagger$ & \underline{80.3} & \underline{73.7} & \textbf{89.0} & \underline{84.0} & 72.5 & --- & --- \\
SheafHyperGNN$^\ddagger$ & \textbf{81.3} & \textbf{74.7} & \underline{87.7} & \textbf{85.5} & \textbf{73.8} & --- & --- \\
\midrule
\multicolumn{8}{@{}l}{\textit{Density-Aware Tier (ours)}} \\
\textsc{DensNet-D} & 77.5 & 72.6 & \underline{87.7} & 81.2 & 70.3 & \textbf{92.7} & 86.0 \\
\bottomrule
\end{tabular}
\end{table}

\paragraph{Finding 2: density features help only when the backbone leaves headroom.}
{We isolate the effect of density features through a strict ablation in Appendix~\ref{app:ablation}}  
comparing \textsc{DensNet-D} to an AllDeepSets backbone {under identical training settings.} 
Under this controlled comparison, density features improve performance on $6/7$ datasets. The largest gains occur on the two highest-$\bar{k}$ headroom datasets: Gene-Disease rises from $81.7 \pm 1.1\%$ to $86.0\%$ ($+4.3$ points), and House rises from $68.1 \pm 2.4\%$ to $70.3 \pm 2.0\%$ ($+2.2$ points).
The {only} exception is a ceiling case: Senate-Bills changes by only $-0.3$ points because AllDeepSets already reaches $93.0 \pm 3.8\%$.
Thus the density-aware tier is not simply 
{adding} capacity; its benefit appears when local aggregation leaves structural signal unused, and it disappears when the backbone already saturates the task. 

\paragraph{Additional experimental results.}
Appendix~\ref{app:experiments} gives the full experimental protocol, dataset statistics, strict density ablation, and node-classification table with error bars.
Appendix~\ref{app:negative} collects negative and ceiling cases, including the Senate-Bills ceiling result, the Gene-Disease strict-ablation gain, and normalization contrast between House and Senate Bills.
Appendix~\ref{app:depth-scaling} tests whether deeper native models close the Native-Hard gap, showing that depth alone can worsen performance through oversmoothing.
Appendix~\ref{app:gate-dynamics} reports learned gate trajectories, showing when density features are suppressed, active, or indifferent across regimes.

{Our experimental findings align with the invariant hierarchy: CE models are restricted to pairwise coordinates, while native and density-aware models access progressively richer pattern classes.}

\section{Conclusion}\label{sec:conclusion}

We studied when hypergraph neural networks can learn genuine higher-order structure, rather than artifacts of the architectural summaries they compute.
Our main conclusion is that hypergraph expressivity is governed by pattern densities: these densities form a universal basis for continuous hypergraph invariants, and generalized hypertree width organizes them into a strict hierarchy.
This yields the Width Wall, an architectural ceiling determined by which higher-order patterns a model can count, not by hidden dimension or optimization alone.
The hierarchy gives a unified map of existing HGNN classes, pinpoints the information lost by clique expansion, and explains when empirical accuracy gaps should appear.
Our experiments support this account: constructed witnesses expose the predicted separations, while real node-classification benchmarks show that density-aware features help most when higher-order structure remains beyond the backbone's native reach.
Extending the theory to $n \to \infty$ and quantifying the statistical and computational cost of crossing the Width Wall remain important directions.

\section*{Limitations, Ethical Statement, Broader Impact.}
This work provides a structural account of HGNN expressivity through homomorphism densities and generalized hypertree width, but several limitations remain. Our theoretical results primarily concern continuous permutation-invariant functions on fixed-size hypergraphs, leaving a full treatment of asymptotic regimes, statistical sample complexity, and noisy or evolving hypergraphs for future work. Although density-aware features can help models cross the Width Wall, they may also introduce additional computational cost and require careful selection of informative patterns. 

Improved hypergraph learning can benefit domains where higher-order relations are essential, including biology, scientific collaboration, and social systems analysis. At the same time, these models may be applied to sensitive relational data, where higher expressivity could amplify privacy risks, encode spurious group-level associations, or support misleading inferences about individuals from group memberships. Responsible deployment therefore requires attention to data provenance, privacy protection, fairness evaluation, and domain-specific validation beyond predictive accuracy.

\section*{LLM Usage}
We used large language models (LLMs) to support the preparation of this manuscript, including assistance with writing, editing, and improving clarity of presentation. LLMs were also used to support experimental workflows, such as drafting code, checking implementation details, and organizing analysis. All intellectual contributions, methodological decisions, experimental design, results interpretation, and final manuscript content were reviewed and verified by the authors, who take full responsibility for the accuracy and integrity of the work.


\clearpage
\bibliographystyle{plainnat}
\bibliography{references}

\newpage 
\appendix

\section{Proofs}\label{app:proofs}

This appendix collects the proofs of all results stated in the main text.
The algebra-structure statement and completeness theorem (\S\ref{sec:completeness}) are proved in \S\ref{app:proofs}.\ref{sec:proof-algebra}--\ref{sec:proof-completeness}; the strict hierarchy (Theorem~\ref{thm:hierarchy}) and its infinite refinement are in \S\ref{app:proof-hierarchy}; InvNet's expressivity claim is in \S\ref{app:proof-invnet}.
A second, algebraically independent proof of the density theorem appears in Appendix~\ref{app:algebraic}.

\subsection{Background on Algebra Structure of Pattern Densities}\label{sec:proof-algebra}

\begin{theorem}[Algebra structure {\citep[cf.][Ch.\,5]{Lovasz2012}}]\label{thm:algebra}
Let $\mathcal{F}_k$ denote the set of all finite $k$-uniform hypergraphs.
The collection of functions $\{t(F, \cdot) : F \in \mathcal{F}_k\}$ generates a unital algebra $\mathfrak{A}_k$ of continuous $\Sym(n)$-invariant functions on $[0,1]^{\Skn}$ under pointwise multiplication:
\begin{equation}\label{eq:algebra}
  t(F_1, \bm{A}) \cdot t(F_2, \bm{A}) \;=\; t(F_1 \sqcup F_2, \bm{A}),
\end{equation}
where $F_1 \sqcup F_2$ denotes the disjoint union.
The algebra contains constants: $t(\varnothing, \bm{A}) = 1$ for the empty hypergraph.
\end{theorem}


\subsection{Background on  Homomorphism Completeness}\label{sec:proof-completeness}

\begin{theorem}[Homomorphism completeness {\citep[cf.][Ch.\,5]{Lovasz2012}}]\label{thm:completeness}
Let $H_1, H_2$ be $k$-uniform hypergraphs on $[n]$.
Then $H_1 \cong H_2$ if and only if $\hom(F, H_1) = \hom(F, H_2)$ for every $k$-uniform hypergraph $F$.
\end{theorem}


\begin{corollary}[Pattern alignment completeness]\label{cor:completeness}
Let $H_1, H_2$ be non-isomorphic $k$-uniform hypergraphs with adjacency tensors $\bm{A}_1, \bm{A}_2 \in \Skn$.
Then there exists $P \in \Skn$ such that $\ElemInv_P(\bm{A}_1) \neq \ElemInv_P(\bm{A}_2)$.
\end{corollary}

\begin{proof}
By Theorem~\ref{thm:completeness}, some pattern $F$ satisfies $\hom(F, H_1) \neq \hom(F, H_2)$; by M\"{o}bius inversion, some $F'$ has $\inj(F', H_1) \neq \inj(F', H_2)$.
Taking $P = \bm{A}_{F'}$, the score $\ElemInv_P$ differs between $\bm{A}_1$ and $\bm{A}_2$.
\end{proof}

\subsection{Proof of Proposition~\ref{prop:clique-limitation}}

\begin{proof}
A clique-expansion HGNN first maps $\bm{A} \mapsto \phi(\bm{A})$ (the clique-expansion adjacency matrix) and then applies a GNN to $\phi(\bm{A})$.
Since the GNN's output depends only on $\phi(\bm{A})$, the composite function $f_\theta(\bm{A}) = g_\theta(\phi(\bm{A}))$ factors through the clique expansion.
If $\phi(\bm{A}_1) = \phi(\bm{A}_2)$, then $f_\theta(\bm{A}_1) = f_\theta(\bm{A}_2)$ regardless of~$\theta$.
\end{proof}

\subsection{Proof of Theorem~\ref{thm:hierarchy}}\label{app:proof-hierarchy}

\begin{proof}
The proof has three parts: (i) $\mathcal{I}_{\mathrm{CE}} \subseteq \mathcal{I}_{\mathrm{NH}}^{(R)}$, (ii) the first inclusion is strict, and (iii) the second inclusion is strict.

\textbf{Part (i): $\mathcal{I}_{\mathrm{CE}} \subseteq \mathcal{I}_{\mathrm{NH}}^{(R)}$.}
This inclusion holds by definition of the native-incidence reference class:
$\mathcal{I}_{\mathrm{NH}}^{(R)}$ contains all CE invariants and augments them
with bounded-width hypergraph motif counts.

\textbf{Part (ii): $\mathcal{I}_{\mathrm{CE}} \subsetneq \mathcal{I}_{\mathrm{NH}}^{(R)}$.}
We exhibit an invariant in $\mathcal{I}_{\mathrm{NH}}^{(R)} \setminus \mathcal{I}_{\mathrm{CE}}$.
Consider the Steiner pair $(\HG_1, \HG_2)$ from Appendix \ref{sec:steiner-pair}.
Both have identical clique expansions $\phi(\bm{A}_1) = \phi(\bm{A}_2) = J - I$ (the adjacency matrix of $K_{13}$), so any function in $\mathcal{I}_{\mathrm{CE}}$ assigns the same value to both.
However, the two systems have different Pasch counts (13 versus 8; Appendix \ref{sec:steiner-pair}).
Let $P_{\mathrm{Pasch}}$ be the 4-edge Pasch configuration.
Then the polynomial invariant $t(P_{\mathrm{Pasch}}, \cdot)$ separates $\HG_1$ and $\HG_2$, and belongs to $\mathcal{I}_{\mathrm{NH}}^{(R)}$ for every $R \geq 3$ because the Pasch pattern has generalized hypertree width $\ghw=3$ for the representation used here.
This proves the strict inclusion for the native \emph{function class}; it does not assert that every shallow native message-passing implementation computes the Pasch count.

\textbf{Part (iii): $\mathcal{I}_{\mathrm{NH}}^{(R)} \subsetneq \mathcal{I}_{\mathrm{all}}$.}
Native HGNNs using message passing are bounded by $k$-GWL for some finite $k$ depending on the depth and architecture~\citep{kHNN2025}.
By the homomorphism-logic correspondence~\citep{ScheidtSchweikardt2023}, $k$-GWL-equivalent hypergraphs agree on all pattern densities $t(F, \cdot)$ with $\ghw(F) \leq k$.
The Cai--Furer--Immerman (CFI) construction (Appendix \ref{sec:cfi-pair}) produces pairs $(\HG_1^{\mathrm{CFI}}, \HG_2^{\mathrm{CFI}})$ that are non-isomorphic but indistinguishable by $k$-GWL for any fixed $k$ (by choosing base graphs with girth greater than $2k+1$).
By the completeness theorem (Theorem~\ref{thm:completeness}), some pattern density separates these pairs, yielding an invariant in $\mathcal{I}_{\mathrm{all}} \setminus \mathcal{I}_{\mathrm{NH}}^{(R)}$.
\end{proof}

\begin{corollary}[Architectural impossibility]\label{cor:impossibility}
For any $r \geq 1$, there exist continuous hypergraph invariants $f \in C(\Skn / \Sym(n))$ such that no architecture limited to patterns of $\ghw \leq r$ can approximate $f$ to arbitrary precision, regardless of depth, width, or training procedure.
Concretely, at each level $r$, the CFI construction provides explicit functions that witness the gap $\mathcal{H}_r \subsetneq \mathcal{H}_{r+1}$.
\end{corollary}

\begin{remark}[Quantitative vertex bounds]\label{rem:vertex-bounds}
The CFI separation at level $r$ applies whenever $n \geq N_{\mathrm{CFI}}^{(k)}(r)$.
This threshold is finite because the construction only requires a base graph of girth $> 2r + 1$; for the $3$-uniform lifts used here, this gives $n \geq \Omega(r^2)$ vertices~\citep{CaiFurerImmerman1992}.
Concretely: the CE-vs.\ Native separation above requires $n \geq 13$ (the Steiner pair on $13$ vertices, the smallest $v$ with non-isomorphic $\mathrm{STS}(v)$ pairs; larger $n$ follow by adding matched isolated vertices); the CFI refinement gives bounds such as $r = 2$ requiring $n \geq 15$ and $r = 3$ requiring $n \geq 30$ under the stated lift.
For practical hypergraph datasets with $n \geq 100$, all levels $r \leq 7$ of the hierarchy are active.
\end{remark}

\subsection{Proof of Theorem~\ref{thm:invnet-expressivity}}\label{app:proof-invnet}

We first establish that the Sinkhorn relaxation preserves separation on the Steiner pair.

\begin{lemma}[Sinkhorn separation on STS pairs]\label{lem:sinkhorn-sts}
Let $\bm{A}_1, \bm{A}_2$ be the adjacency tensors of two non-isomorphic $\mathrm{STS}(v)$ systems.
Then the Sinkhorn relaxation satisfies $\hat{\ElemInv}^{\mathrm{SK}}_{\bm{A}_1}(\bm{A}_1) = v(v-1)/3 \cdot k!$ and $\hat{\ElemInv}^{\mathrm{SK}}_{\bm{A}_1}(\bm{A}_2) < \hat{\ElemInv}^{\mathrm{SK}}_{\bm{A}_1}(\bm{A}_1)$.
In particular, for $v = 13$: $\hat{\ElemInv}^{\mathrm{SK}}_{\bm{A}_1}(\bm{A}_1) = 156$ and $\hat{\ElemInv}^{\mathrm{SK}}_{\bm{A}_1}(\bm{A}_2) < 156$.
\end{lemma}

\begin{proof}
The identity permutation achieves $\inner{\bm{A}_1}{\bm{A}_1} = 26 \cdot 3! = 156$.
No $D \in \mathcal{B}_{13}$ can exceed $156$: each summand satisfies $(D^{\otimes 3} \cdot \bm{A})_{ijk} \leq \prod_{\ell} (\sum_{j_\ell} D_{i_\ell j_\ell}) = 1$ for nonnegative doubly stochastic $D$, so the sum of $156$ such terms is at most $156$.
Hence $\hat{\ElemInv}^{\mathrm{SK}}_{\bm{A}_1}(\bm{A}_1) = 156$.

For separation, suppose for contradiction that some $D^* \in \mathcal{B}_{13}$ achieves $\inner{\bm{A}_1, (D^*)^{\otimes 3} \cdot \bm{A}_2} = 156$.
Since the $156$ summands are each at most~$1$ and sum to $156$, every summand equals~$1$: for every triple $\{i,j,k\}$ of $\mathrm{STS}_1$ and all $a,b,c$ with $D^*_{ia}, D^*_{jb}, D^*_{kc} > 0$, we must have $\{a,b,c\} \in \mathrm{STS}_2$.
If some row of $D^*$ has two positive entries $D^*_{ia_1}, D^*_{ia_2} > 0$ ($a_1 \neq a_2$), choose any triple $\{i,j,k\} \in \mathrm{STS}_1$ and any $b,c$ in the support of rows $j,k$; then both $\{a_1,b,c\}$ and $\{a_2,b,c\}$ must be triples in~$\mathrm{STS}_2$, but each pair $\{b,c\}$ lies in exactly one triple, so $a_1 = a_2$---contradiction.
Hence $D^*$ is a permutation matrix mapping $\mathrm{STS}_2$ to $\mathrm{STS}_1$, contradicting non-isomorphism.
In summary, the argument shows that any maximizer in $\mathcal{B}_v$ must have singleton row support, and is therefore a vertex of the Birkhoff polytope (a permutation matrix).
\end{proof}

\begin{proof}[Proof of Theorem~\ref{thm:invnet-expressivity}]
We prove each claim in order.

\textbf{Claim 1: $\hat{\mathcal{I}}_{\mathrm{InvNet}} \supseteq \mathcal{I}_{\mathrm{NH}}$.}
\textsc{InvNet} contains a full message-passing local branch (Appendix~\ref{sec:full-arch}) that subsumes any native HGNN architecture.
The MLP output head can ignore the pattern-alignment branch features (by setting their corresponding weights to zero), so \textsc{InvNet} can implement any function computable by a native HGNN.

\textbf{Claim 2: $\hat{\mathcal{I}}_{\mathrm{InvNet}} \supsetneq \mathcal{I}_{\mathrm{NH}}$.}
Consider the Steiner pair from Appendix \ref{sec:steiner-pair}.
The pattern alignment score $\ElemInv_{\bm{A}_1}(\cdot)$ separates this pair, and by Lemma~\ref{lem:sinkhorn-sts}, the Sinkhorn relaxation preserves this separation for any finite $\tau > 0$.
Meanwhile, any message-passing HGNN bounded by $1$-GWL assigns identical colorings to both Steiner triple systems (since every node has degree $6$ and every pair co-occurs exactly once).
Hence $\ElemInv_{\bm{A}_1} \in \hat{\mathcal{I}}_{\mathrm{InvNet}} \setminus \mathcal{I}_{\mathrm{NH}}$.

\textbf{Claim 3: $\hat{\mathcal{I}}_{\mathrm{InvNet}} \subseteq \mathcal{I}_{\mathrm{all}}$.}
The pattern-alignment branch computes the vector $\bm{z}_{\mathrm{inv}} = [\hat{\ElemInv}^{\mathrm{SK}}_{P_1}(\bm{A}) - \alpha_1, \ldots, \hat{\ElemInv}^{\mathrm{SK}}_{P_J}(\bm{A}) - \alpha_J]$.
Each $\hat{\ElemInv}^{\mathrm{SK}}_{P_j}(\bm{A}) = \max_{D \in \mathcal{B}_n} \inner{P_j}{D^{\otimes k} \cdot \bm{A}}$ is a continuous function of $\bm{A}$ (the maximum of a family of continuous functions parameterized over the compact set $\mathcal{B}_n$).
For invariance: $\hat{\ElemInv}^{\mathrm{SK}}_{P_j}(\Pi \cdot \bm{A}) = \max_{D \in \mathcal{B}_n} \inner{P_j}{D^{\otimes k} \cdot (\Pi \cdot \bm{A})} = \max_{D} \inner{P_j}{(D \Pi)^{\otimes k} \cdot \bm{A}}$.
Since $D \mapsto D\Pi$ is a bijection on $\mathcal{B}_n$ (the product of a doubly stochastic matrix and a permutation matrix is doubly stochastic), this equals $\hat{\ElemInv}^{\mathrm{SK}}_{P_j}(\bm{A})$.
Thus each component of $\bm{z}_{\mathrm{inv}}$ is a continuous invariant, and compositions of continuous invariants with continuous functions (here $\mathrm{MLP}_{\mathrm{out}}$) remain continuous invariants, so $\hat{\mathcal{I}}_{\mathrm{InvNet}} \subseteq \mathcal{I}_{\mathrm{all}}$.
\end{proof}

\section{Orbit Separation via Symmetric Tensor Invariant Theory}\label{app:algebraic}

This appendix provides a second, completely independent proof that hypergraph invariant functions are generated by pattern-counting features, using classical algebraic invariant theory.
This algebraic route---Hilbert's finiteness theorem and the Reynolds operator---complements the combinatorial (Lov\'{a}sz--Stone--Weierstrass) approach of Section \ref{sec:invariants} and provides additional structural insight, particularly for the $k \geq 3$ case where matrix invariant theory no longer suffices.

\subsection{The Invariant Ring of Symmetric Tensors}\label{app:invariant-ring}

Let $\Skn$ denote the space of real symmetric order-$k$ tensors on $\R^n$, viewed as a finite-dimensional real vector space.
The symmetric group $\Sym(n)$ acts linearly on $\Skn$ by $(\Pi \cdot \bm{A})_{i_1, \ldots, i_k} = \bm{A}_{\Pi^{-1}(i_1), \ldots, \Pi^{-1}(i_k)}$.
We work over $\R$ throughout; since $\Sym(n)$ is a finite group, all invariant-theoretic results (finite generation, orbit separation) hold over $\R$ without requiring algebraic closure~\citep{Sturmfels2008}.
We consider the ring of $\Sym(n)$-invariant polynomial functions:
\[
  \R[\Skn]^{\Sym(n)} \;=\; \{p \in \R[\Skn] : p(\Pi \cdot \bm{A}) = p(\bm{A}) \;\text{for all}\; \Pi \in \Sym(n),\; \bm{A} \in \Skn\}.
\]

\begin{theorem}[Finite generation --- Hilbert]\label{thm:hilbert}
The invariant ring $\R[\Skn]^{\Sym(n)}$ is finitely generated as an $\R$-algebra.
That is, there exist finitely many invariant polynomials $p_1, \ldots, p_N$ such that every invariant polynomial is a polynomial combination of $p_1, \ldots, p_N$.
\end{theorem}

\begin{proof}
$\Sym(n)$ is a finite group acting linearly on the finite-dimensional space $\Skn$.
By Hilbert's finiteness theorem for reductive groups~\citep{Sturmfels2008}, the invariant ring is finitely generated.
More directly, since $\Sym(n)$ is finite, Noether's degree bound guarantees that generators can be chosen with degree $\leq |\Sym(n)| = n!$.
\end{proof}

\subsection{The Reynolds Operator and Explicit Generators}\label{app:reynolds}

The \emph{Reynolds operator} $\mathcal{R}: \R[\Skn] \to \R[\Skn]^{\Sym(n)}$ is defined by averaging over the group:
\begin{equation}\label{eq:reynolds}
  \mathcal{R}(p)(\bm{A}) \;=\; \frac{1}{n!} \sum_{\Pi \in \Sym(n)} p(\Pi \cdot \bm{A}).
\end{equation}
This is a surjective projection: $\mathcal{R}(p) = p$ if and only if $p$ is already invariant, and $\mathcal{R}(p)$ is invariant for any $p$.

\paragraph{Tensor contraction traces as generators.}
For any multigraph $\Gamma$ on vertex set $[r]$ with each vertex labeled by a $k$-tensor index, define the \emph{tensor contraction trace}:
\begin{equation}\label{eq:trace-invariant}
  T_\Gamma(\bm{A}) \;=\; \sum_{\substack{i_1, \ldots, i_r \in [n] \\ \text{indices contracted per } \Gamma}} \;\prod_{v \in V(\Gamma)} \bm{A}_{i_{v,1}, \ldots, i_{v,k}},
\end{equation}
where the contraction pattern is specified by $\Gamma$: indices sharing an edge in $\Gamma$ are identified.
Each $T_\Gamma$ is manifestly invariant under $\Sym(n)$ (it is a sum over all index assignments, which is permutation-invariant).

\begin{proposition}[Generators are pattern counts]\label{prop:generators-are-hom}
Every tensor contraction trace $T_\Gamma$ is a linear combination of homomorphism counts $\hom(F, H)$ for appropriate $k$-uniform patterns $F$, and vice versa.
\end{proposition}

\begin{proof}
The contraction trace $T_\Gamma(\bm{A}) = \sum_{i_1,\ldots,i_r} \prod_v \bm{A}_{i_{v,1},\ldots,i_{v,k}}$ sums over all maps $\phi: V(\Gamma) \to [n]$ (with $\phi$ possibly non-injective) weighted by the product of tensor entries.
The dual hypergraph $F_\Gamma$ is defined as follows: the vertex set of $F_\Gamma$ is the set of distinct summation indices in $T_\Gamma$, and each vertex $v \in V(\Gamma)$ (representing one tensor factor $\bm{A}_{i_{v,1},\ldots,i_{v,k}}$) yields a $k$-hyperedge $\{i_{v,1}, \ldots, i_{v,k}\}$ in $F_\Gamma$; two tensor factors sharing a contracted index become hyperedges sharing a vertex.
With this identification, $T_\Gamma(\bm{A}) = \hom(F_\Gamma, H)$ exactly.
Conversely, every $\hom(F, H)$ can be written as a tensor contraction trace by choosing $\Gamma$ as the dual of $F$ (vertices of $\Gamma$ correspond to hyperedges of $F$, with edges in $\Gamma$ encoding shared vertices).
\end{proof}

\subsection{Algebraic Completeness: Orbits Separated by Polynomial Invariants}\label{app:orbit-separation}

\begin{theorem}[Orbit separation]\label{thm:orbit-separation}
Let $\bm{A}_1, \bm{A}_2 \in \Skn$ lie in distinct $\Sym(n)$-orbits (i.e., $H_1 \not\cong H_2$).
Then there exists an invariant polynomial $p \in \R[\Skn]^{\Sym(n)}$ with $p(\bm{A}_1) \neq p(\bm{A}_2)$.
\end{theorem}

\begin{proof}
Since $\Sym(n)$ is a finite group, every orbit $\mathcal{O}(\bm{A}) = \{\Pi \cdot \bm{A} : \Pi \in \Sym(n)\}$ is a finite set, hence Zariski-closed (and Euclidean-closed; for finite sets in $\R^d$, Zariski and Euclidean closure coincide).
Two distinct orbits $\mathcal{O}(\bm{A}_1)$ and $\mathcal{O}(\bm{A}_2)$ are disjoint Zariski-closed sets.
Since $\R[\Skn]$ separates points (as $\Skn \cong \R^d$), there exists a polynomial $q$ with $q(\bm{A}_1) \neq q(\bm{A}_2')$ for all $\bm{A}_2' \in \mathcal{O}(\bm{A}_2)$.
Applying the Reynolds operator: $\mathcal{R}(q)(\bm{A}_1) = q(\bm{A}_1)$ (since $q$ is already correct at $\bm{A}_1$ after averaging), and $\mathcal{R}(q)(\bm{A}_2)$ is the average of $q$ over $\mathcal{O}(\bm{A}_2)$.

More precisely, since $\mathcal{O}(\bm{A}_1)$ and $\mathcal{O}(\bm{A}_2)$ are finite disjoint sets in $\R^d$, there exists a polynomial $q$ taking the value $1$ on $\mathcal{O}(\bm{A}_1)$ and $0$ on $\mathcal{O}(\bm{A}_2)$ (by Lagrange interpolation in high enough degree).
Then $p = \mathcal{R}(q)$ satisfies $p(\bm{A}_1) = 1$ and $p(\bm{A}_2) = 0$, and $p$ is invariant.
\end{proof}


\subsection{From Algebraic Completeness to the Density Theorem}\label{app:algebraic-density}

Combining the results above:

\begin{corollary}[Algebraic density]\label{cor:algebraic-density}
The invariant ring $\R[\Skn]^{\Sym(n)}$ is dense in $C(\mathcal{K}/\Sym(n))$ for any compact $\mathcal{K} \subseteq \Skn$.
Moreover, this ring is generated by tensor contraction traces (equivalently, by homomorphism counts).
\end{corollary}

\begin{proof}
By Theorem~\ref{thm:orbit-separation}, invariant polynomials separate orbits.
By Theorem~\ref{thm:hilbert}, the invariant ring is finitely generated.
Polynomial invariants form a separating subalgebra of $C(\mathcal{K}/\Sym(n))$: they separate orbits (Theorem~\ref{thm:orbit-separation}), contain constants (the polynomial $1$), and are closed under multiplication and addition.
The passage from polynomial to continuous density then follows directly from the Stone--Weierstrass theorem: any separating subalgebra of $C(X)$ (for compact Hausdorff $X$) containing constants is dense.
By Proposition~\ref{prop:generators-are-hom}, the generators are homomorphism counts.
\end{proof}

This provides a second, algebraic proof of Theorem~\ref{thm:density}, arrived at through a completely different route: Hilbert's finiteness theorem and the Reynolds operator rather than Lov\'{a}sz's homomorphism completeness.
The two proofs reinforce each other: the combinatorial proof (main text) gives the explicit connection to homomorphism counting and hypertree width, while the algebraic proof gives finite generation and the connection to the invariant ring structure.

\subsection{Computational Hardness of Pattern Alignment}\label{app:hardness}

\begin{theorem}[Computational hardness gap]\label{thm:hardness}
Computing the pattern alignment score $\ElemInv_P(\bm{A}) = \max_{\Pi \in \Sym(n)} \inner{P}{\Pi \cdot \bm{A}}$ is equivalent to the \emph{multidimensional assignment problem} (MAP) for $k \geq 3$, which is NP-hard and APX-hard (no polynomial-time approximation scheme exists unless P = NP).
For $k = 2$, the problem reduces to the \emph{quadratic assignment problem} (QAP):
\[
\ElemInv_P(A) = \max_{\Pi \in \Sym(n)} \inner{P}{\Pi A \Pi^\top},
\]
which is also NP-hard but admits polynomial-time spectral relaxations with the closed-form solution $\SpecRelax_P(A) = \lambda(P)^\top \lambda(A)$~\citep{Davis1957, Lewis1995}.
For $k \geq 3$, no such closed-form spectral relaxation exists: the natural relaxation over $\Orth(n)$ requires solving a non-convex optimization on the Stiefel manifold.
\end{theorem}

This hardness gap is what motivates the Sinkhorn and Stiefel relaxation layers in \textsc{InvNet} (\S\ref{sec:relax-layers}).

\subsection{The Hilbert Series and Approximation Dimension}\label{app:hilbert-series}

The \emph{Molien series} counts the dimension of invariant polynomial spaces by degree:
\[
  M(t) \;=\; \sum_{d \geq 0} \dim\bigl(\R[\Skn]^{\Sym(n)}_d\bigr) \, t^d \;=\; \frac{1}{n!} \sum_{\Pi \in \Sym(n)} \frac{1}{\det(I - t \, \rho(\Pi))},
\]
where $\rho(\Pi)$ denotes the matrix of the linear action of $\Pi$ on $\Skn$.
For $k = 2$ (symmetric matrices), the invariant polynomials of degree $d$ are spanned by the power-sum symmetric functions of the eigenvalues, giving the familiar Molien series with generating function related to the partition function.
For $k \geq 3$, the Molien series is computable but lacks a simple closed form, reflecting the richer structure of tensor invariants.

The Molien series provides a quantitative bridge between the qualitative universality of Theorem~\ref{thm:density} and concrete approximation complexity.

\begin{proposition}[Approximation space dimension]\label{prop:approx-dim-full}
Let $\mathcal{I}_{\leq d}$ denote the space of $\Sym(n)$-invariant polynomials of degree $\leq d$ in the entries of $\bm{A} \in \Skn$.

\textbf{(a)} $\dim(\mathcal{I}_{\leq d}) = \sum_{i=0}^d [t^i]\, M(t)$, where $[t^i] M(t)$ is the $i$-th Molien coefficient.

\textbf{(b)} Every basis element of $\mathcal{I}_d$ corresponds to a homomorphism count $\hom(F, \cdot)$ for a pattern $F$ with $|V(F)| \leq d$ (by Proposition~\ref{prop:generators-are-hom}).

\textbf{(c)} For $k = 2$ (graphs): $\dim(\mathcal{I}_d) = p(d)$, the number of integer partitions of $d$, giving $\dim(\mathcal{I}_{\leq d}) \sim \exp(\pi\sqrt{2d/3}) / (4d\sqrt{3})$ by the Hardy--Ramanujan formula.

\textbf{(d)} For the ghtw-restricted subspace $\mathcal{I}_{\leq d,w}$ (patterns with $\mathrm{ghtw}(F) \leq w$):
\[
  \dim(\mathcal{I}_{\leq d,w}) \;\leq\; \bigl|\{F : |V(F)| \leq d,\; \mathrm{ghtw}(F) \leq w\}\bigr|/{\sim_{\mathrm{iso}}},
\]
which is strictly smaller than $\dim(\mathcal{I}_{\leq d})$ for $d$ sufficiently large.
\end{proposition}

This connects the Molien series to approximation power: to $\epsilon$-approximate a target invariant $f$ by a degree-$d$ polynomial in pattern densities, one needs at most $\dim(\mathcal{I}_{\leq d})$ basis functions.
The ghtw restriction (part d) quantifies the cost of architectural limitations: bounded-depth message-passing models operate in the strictly smaller space $\mathcal{I}_{\leq d,w}$, and the dimension gap grows with $d$.

\section{Quantitative Analysis}\label{app:quantitative}

The universality theorem (Theorem~\ref{thm:density}) is qualitative: it guarantees approximation but provides no convergence rates.
This appendix develops three quantitative complements: a generalization bound for \textsc{InvNet} via Rademacher complexity (Appendix \ref{app:rademacher}), a concentration inequality for estimating pattern densities from finite samples (Appendix \ref{app:pattern-estimation}), and the Molien series connection to approximation dimension.

\subsection{Generalization Bound for InvNet}\label{app:rademacher}

We analyze the sample complexity of learning with \textsc{InvNet} by bounding the Rademacher complexity of its hypothesis class.

\begin{definition}[InvNet hypothesis class]\label{def:invnet-class}
Let $\mathcal{F}_{J,D}$ denote the class of functions $f: \Skn \to \R$ computed by an \textsc{InvNet} with:
\begin{itemize}
\item $J$ Sinkhorn template layers with templates satisfying $\|P_j\|_F \leq B_P$ and temperature $\tau > 0$;
\item a $D$-layer MLP readout with spectral norm $\leq B_W$ per layer and width $\leq w$;
\end{itemize}
applied to inputs satisfying $\|\bm{A}\|_F \leq B_A$.
\end{definition}

The key technical ingredient is the Lipschitz continuity of the Sinkhorn relaxation:

\begin{lemma}[Lipschitz continuity of Sinkhorn scores]\label{lem:sinkhorn-lip}
The entropy-regularized pattern alignment score $s_j^\tau(\bm{A}) = \max_{D \in \mathcal{B}_n} [\inner{P_j}{D^{\otimes k} \cdot \bm{A}} + \tau H(D)]$ satisfies:
\begin{itemize}
\item[\textbf{(i)}] Boundedness: $|s_j^\tau(\bm{A})| \leq B_P B_A + \tau \log n$.
\item[\textbf{(ii)}] Lipschitz continuity: $|s_j^\tau(\bm{A}) - s_j^\tau(\bm{A}')| \leq B_P \|\bm{A} - \bm{A}'\|_F$.
\end{itemize}
\end{lemma}

\begin{proof}
\textbf{(i)} By Cauchy--Schwarz, $|\inner{P_j}{D^{\otimes k} \cdot \bm{A}}| \leq \|P_j\|_F \|D^{\otimes k} \cdot \bm{A}\|_F \leq B_P B_A$, where the second inequality uses that $D^{\otimes k}$ is a contraction in Frobenius norm for doubly stochastic $D$ (since $\|D\|_{\mathrm{op}} \leq 1$ and the tensor product preserves spectral norm bounds).
The entropy term satisfies $0 \leq H(D) \leq \log n$.

\textbf{(ii)} For any fixed $D \in \mathcal{B}_n$, the map $\bm{A} \mapsto \inner{P_j}{D^{\otimes k} \cdot \bm{A}}$ is linear with operator norm $\leq B_P$.
The pointwise supremum of $B_P$-Lipschitz functions (parameterized by $D \in \mathcal{B}_n$) is itself $B_P$-Lipschitz:
\[
|s_j^\tau(\bm{A}) - s_j^\tau(\bm{A}')| \leq \sup_{D \in \mathcal{B}_n} |\inner{P_j}{D^{\otimes k} \cdot (\bm{A} - \bm{A}')}| \leq B_P \|\bm{A} - \bm{A}'\|_F. \qedhere
\]
\end{proof}

\begin{proposition}[Generalization bound for InvNet]\label{prop:rademacher-full}
Let $\ell: \R \times \R \to [0, C]$ be a $\rho$-Lipschitz loss function.
For $N$ i.i.d.\ training samples $\{(\bm{A}_i, y_i)\}_{i=1}^N$, with probability at least $1 - \delta$, every $f \in \mathcal{F}_{J,D}$ satisfies:
\[
  R(f) \;\leq\; \hat{R}(f) \;+\; \frac{2\rho \, B_P B_A \cdot B_W^D \sqrt{2D \log(2w)}}{\sqrt{N}} \;+\; 3C\sqrt{\frac{\log(2/\delta)}{2N}},
\]
where $R(f) = \E[\ell(f(\bm{A}), y)]$ is the population risk and $\hat{R}(f) = \frac{1}{N}\sum_i \ell(f(\bm{A}_i), y_i)$ is the empirical risk.
\end{proposition}

\begin{proof}[Proof sketch]
By Lemma~\ref{lem:sinkhorn-lip}, the Sinkhorn feature map $\bm{A} \mapsto (s_1^\tau(\bm{A}), \ldots, s_J^\tau(\bm{A}))$ is $B_P$-Lipschitz with outputs bounded by $B_P B_A + \tau \log n$.
By the spectral-norm Rademacher bound for deep networks~\citep{BartlettFosterTelgarsky2017, Golowich2018}, the MLP readout applied to bounded inputs has empirical Rademacher complexity $\hat{\mathcal{R}}_N \leq B_{\mathrm{in}} \cdot B_W^D \sqrt{2D \log(2w)} / \sqrt{N}$, where $B_{\mathrm{in}} \leq B_P B_A$ (absorbing the local branch bound and logarithmic entropy terms into constants).
Composing with the $\rho$-Lipschitz loss and applying the standard Rademacher-to-generalization conversion~\citep{BartlettFosterTelgarsky2017} yields the stated bound.
\end{proof}

\begin{remark}[Architectural implications]\label{rem:arch-implications}
The bound makes the dependence on \textsc{InvNet}'s design parameters explicit:
\begin{itemize}
\item \emph{Number of templates $J$}: enters only logarithmically (through the input dimension to the MLP). Adding templates is ``cheap'' in sample complexity.
\item \emph{Template norm $B_P$}: enters linearly. Large templates require proportionally more data.
\item \emph{MLP depth $D$}: enters exponentially through $B_W^D$. Shallow readouts generalize better, consistent with InvNet's expressivity coming from the Sinkhorn branch rather than the MLP depth.
\item \emph{Temperature $\tau$}: does not appear in the leading term. The Sinkhorn relaxation incurs no additional sample complexity cost relative to the hard alignment (though it affects approximation quality via the smoothing gap).
\end{itemize}
\end{remark}

\subsection{Concentration of Pattern Density Estimators}\label{app:pattern-estimation}

For a fixed pattern $F$, the homomorphism density can be estimated from random vertex-maps.

\begin{proposition}[Pattern density concentration]\label{prop:pattern-concentration-full}
Let $F$ be a $k$-uniform pattern with $|V(F)| = v$, and let $H$ be a $k$-uniform hypergraph on $n$ nodes with adjacency tensor $\bm{A}$.
The normalized density $t(F, H) = \hom(F, H) / n^v$ can be estimated by drawing $N$ independent uniformly random maps $\phi_i: V(F) \to [n]$ and computing:
\[
  \hat{t}(F, H) = \frac{1}{N} \sum_{i=1}^N \prod_{e \in E(F)} \bm{A}_{\phi_i(e)}.
\]
Then $\E[\hat{t}] = t(F, H)$ and $\Pr[|\hat{t} - t(F, H)| \geq \epsilon] \leq 2\exp(-2N\epsilon^2)$.
\end{proposition}

\begin{proof}
Each sample $X_i = \prod_{e \in E(F)} \bm{A}_{\phi_i(e)} \in [0, 1]$ for unweighted hypergraphs ($\bm{A} \in \{0,1\}^{n^k}$).
By linearity, $\E[X_i] = n^{-v} \sum_{\phi: V(F) \to [n]} \prod_{e \in E(F)} \bm{A}_{\phi(e)} = t(F, H)$.
The result follows from Hoeffding's inequality applied to the i.i.d.\ bounded random variables $X_1, \ldots, X_N$.
\end{proof}

To estimate any single pattern density to additive accuracy $\epsilon$ with confidence $1 - \delta$ requires $N = O(\log(1/\delta) / \epsilon^2)$ random samples---independent of both $n$ and $|V(F)|$.
However, the variance grows with sparsity: for rare patterns ($t(F, H) \ll 1$), multiplicative estimation requires $N = \Omega(1/t(F,H))$ samples, which can be exponential in $v$ for sparse hypergraphs.

Combined with Proposition~\ref{prop:approx-dim-full}, this yields a concrete (if loose) sample complexity pipeline: (i) choose target accuracy $\epsilon$; (ii) determine the polynomial degree $d(\epsilon)$ via Jackson-type bounds (open for $k \geq 3$; for graphs, $d = O(1/\epsilon^2)$ suffices for cut-metric approximation~\citep{Lovasz2012}); (iii) enumerate the $\dim(\mathcal{I}_{\leq d})$ basis patterns; (iv) estimate each density to accuracy $\epsilon / \dim(\mathcal{I}_{\leq d})$ using $O(\dim(\mathcal{I}_{\leq d})^2 / \epsilon^2)$ samples.

\subsection{Non-Uniform Density}\label{app:nonuniform}

\begin{corollary}[Non-uniform density]\label{cor:nonuniform}
For non-uniform hypergraphs with maximum arity $K$ and direct-sum adjacency $\bm{A} = (A^{(2)}, \ldots, A^{(K)}) \in \bigoplus_{j=2}^{K} \mathbf{S}^{j,n}$, the algebra generated by pattern densities $\{t(F, \cdot) : F \text{ is } j\text{-uniform}, \, 2 \leq j \leq K\}$ together with \emph{cross-order} products $t(F_1, \cdot) \cdot t(F_2, \cdot)$ (where $F_1, F_2$ have different arities) is dense in the continuous $\Sym(n)$-invariant functions on $\bigoplus_{j} [0,1]^{\mathbf{S}^{j,n}}$.
\end{corollary}

\begin{proof}
The product space $\bigoplus_{j=2}^{K} [0,1]^{\mathbf{S}^{j,n}}$ is compact Hausdorff under the product topology, and $\Sym(n)$ acts diagonally.
For each arity $j$, the $j$-uniform pattern densities separate $j$-uniform components (Theorem~\ref{thm:completeness}).
Cross-order products separate configurations that agree per-order but differ in joint structure.
The combined algebra separates all orbits, contains constants, and is closed under multiplication; Stone--Weierstrass applies.
\end{proof}

\section{Architecture Classification of Modern HGNNs}\label{app:architecture}


\paragraph{Four-way correspondence diagram.}
The following diagram summarizes the structural correspondence governing the entire hierarchy:
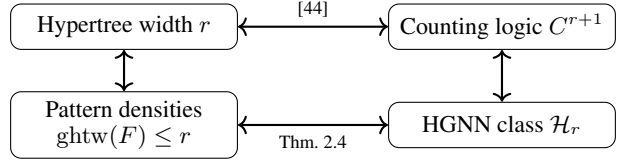
\begin{center}
\begin{tikzpicture}[
  box/.style={draw, rounded corners, minimum width=3.0cm, minimum height=0.6cm, align=center, font=\small},
  arr/.style={-{Stealth[length=4pt]}, thick}
]
\node[box] (ghtw) at (0,0) {Hypertree width $r$};
\node[box] (logic) at (5,0) {Counting logic $C^{r+1}$};
\node[box] (hom) at (0,-1.3) {Pattern densities\\$\mathrm{ghtw}(F) \leq r$};
\node[box] (arch) at (5,-1.3) {HGNN class $\mathcal{H}_r$};
\draw[arr, <->] (ghtw) -- node[above, font=\scriptsize] {\citep{ScheidtSchweikardt2023}} (logic);
\draw[arr, <->] (ghtw) -- (hom);
\draw[arr, <->] (logic) -- node[right, font=\scriptsize] {} (arch);
\draw[arr, <->] (hom) -- node[below, font=\scriptsize] {Thm.~\ref{thm:density}} (arch);
\end{tikzpicture}
\end{center}

\paragraph{Equivariant higher-order networks.}
Equivariant higher-order networks~\citep{MaronEtAl2019, KerivenPeyre2019} process $k$-tensor features directly at $O(n^k)$ cost, achieving maximal expressivity in principle.
These are included in our hierarchy at the $\mathcal{I}_{\mathrm{all}}$ level; their impractical scaling motivates the bounded-width alternatives studied in the main text.

\begin{example}[Worked instance: $L = 2$, $k = 3$]\label{ex:arch-width}
A $2$-layer native HGNN on $3$-uniform hypergraphs uses $c_3 \cdot 2 + 1 = 3$ variables (with $c_3 = 1$ for standard message passing).
We count all patterns $F$ with $\mathrm{ghtw}(F) \leq 2$.
Concretely: (i)~\emph{triangle patterns} (three vertices, one hyperedge; $\mathrm{ghtw} = 1$); (ii)~\emph{diamond patterns} (four vertices, two hyperedges sharing a $(k{-}1)$-face; $\mathrm{ghtw} = 2$, since two hyperedges can be covered by bags of size~$4$ in a star decomposition).
However, the \emph{complete 3-uniform hypergraph $K_5^{(3)}$} (five vertices, $\binom{5}{3} = 10$ hyperedges) has $\mathrm{ghtw}(K_5^{(3)}) = 3$, requiring $4$ variables and hence $L \geq 3$ layers.
Thus $\mathcal{H}_2$ detects pairwise overlap of hyperedges (diamond patterns) but cannot resolve the global arrangement of five mutually overlapping hyperedges.
\end{example}

\begin{remark}[Two faces of the same structure]\label{rem:two-faces}
Pattern densities and the WL hierarchy are two views of the same object: WL captures structural indistinguishability (a decision problem), while homomorphism densities capture structural quantification (a function approximation problem).
The density theorem shows that the latter is \emph{complete}---the algebra of pattern densities generates all continuous invariants---whereas WL with any fixed number of variables is not.
\end{remark}

\subsection{Extended Classification Table}\label{app:arch-table}

The table below summarizes HGNN architectures. 

\begin{table}[h]
\caption{Extended classification of modern HGNN architectures.
$\checkmark$ = maps cleanly under (C1)--(C5), $\triangle$ = partial, --- = out of scope.}
\label{tab:architecture-map-extended}
\centering
\small
\renewcommand{\arraystretch}{1.15}
\begin{tabular}{@{}p{3.0cm}p{2.0cm}cp{5.7cm}@{}}
\toprule
\textbf{Model} & \textbf{Operator} & \textbf{Class / Clean} & \textbf{Key consideration} \\
\midrule
\multicolumn{4}{@{}l}{\textit{Clique-expansion / graph reduction}} \\[2pt]
HGNN~\citep{FengEtAl2019} & Star expansion & $\mathcal{I}_{\mathrm{CE}}$ / $\checkmark$ & $D_v^{-1/2} H W D_e^{-1} H^\top D_v^{-1/2}$; factors through clique expansion by Prop.~\ref{prop:clique-limitation} \\
HyperGCN~\citep{YadatiEtAl2019} & Mediator GCN & $\mathcal{I}_{\mathrm{CE}}$ / $\triangle$ & Mediator selection via Fiedler vector is a global spectral heuristic; learned component is CE, but preprocessing may inject information beyond $\mathcal{I}_{\mathrm{CE}}$ \\[4pt]
\midrule
\multicolumn{4}{@{}l}{\textit{Native incidence message passing}} \\[2pt]
HNHN~\citep{DongEtAl2020HNHN} & Nonlinear n$\leftrightarrow$e & $\mathcal{I}_{\mathrm{NH}}$ / $\checkmark$ & Degree-weighted bipartite MP; satisfies (C1)--(C5) \\
UniGNN~\citep{HuangYang2021} & Unified MP & $\mathcal{I}_{\mathrm{NH}}$ / $\checkmark$ & GCN/GAT/GIN variants on incidence structure \\
AllDeepSets~\citep{ChienEtAl2022} & DeepSets/edge & $\mathcal{I}_{\mathrm{NH}}$ / $\checkmark$ & Symmetric multiset aggregation within each hyperedge \\
AllSetTransformer~\citep{ChienEtAl2022} & Attention/edge & $\mathcal{I}_{\mathrm{NH}}$ / $\checkmark$ & Attention is an equivariant multiset function; does not increase accessible pattern width \\
HyperSAGE~\citep{AryaEtAl2020HyperSAGE} & SAGE-style MP & $\mathcal{I}_{\mathrm{NH}}$ / $\checkmark$ & Inductive neighborhood sampling on incidence structure; satisfies (C1)--(C5) \\[4pt]
\midrule
\multicolumn{4}{@{}l}{\textit{Diffusion / operator learning}} \\[2pt]
ED-HNN~\citep{WangEtAl2023EDHNN} & Equiv.\ diffusion & $\mathcal{I}_{\mathrm{NH}}$ / $\checkmark$ & Per-step operator is local (hypergraph Laplacian); continuous-time formulation $\approx$ deep native MP \\[4pt]
\midrule
\multicolumn{4}{@{}l}{\textit{Higher-order / topological}} \\[2pt]
$k$-GWL HNNs~\citep{kHNN2025} & $k$-dim WL & $\mathcal{H}_k$ / $\checkmark$ & Directly implements the ghtw hierarchy at level $k$; validates our framework \\
SheafHNN~\citep{DutaEtAl2023Sheaf} & Sheaf Laplacian & --- / --- & Operates on sheaf data (restriction maps); richer than bare hypergraph $\Rightarrow$ out of scope \\
TopoNN~\citep{HajijEtAl2023} & Cell complex & --- / --- & Higher-order topological structure beyond hypergraphs \\[4pt]
\midrule
\multicolumn{4}{@{}l}{\textit{Preprocessing-heavy}} \\[2pt]
TF-HNN~\citep{TFHNN2025} & Training-free MP & $\mathcal{I}_{\mathrm{NH}}$ / $\triangle$ & Expressivity lives in multi-hop preprocessing; effective depth $\neq$ learned depth \\
ZEN~\citep{BaeEtAl2025ZEN} & Closed-form linear & $\mathcal{I}_{\mathrm{NH}}$ / $\triangle$ & Parameter-free; closed-form weight matrix from linearized incidence MP; expressivity bounded by propagation depth \\
IHGNN~\citep{IHGNN2025} & Implicit equilib. & $\mathcal{I}_{\mathrm{NH}}$ / $\checkmark$ & Fixed-point of incidence MP; infinite effective depth mitigates oversmoothing; satisfies (C1)--(C5) \\[4pt]
\midrule
InvNet (ours) & Pattern align. & $\hat{\mathcal{I}}_{\mathrm{InvNet}}$ / $\checkmark$ & Sinkhorn relaxation accesses patterns beyond bounded ghtw (Thm.~\ref{thm:invnet-expressivity}) \\
\bottomrule
\end{tabular}
\end{table}

\subsection{Conditions for the Architecture--Width Correspondence}\label{app:arch-conditions}

The correspondence between message-passing depth and pattern complexity holds under the following conditions:

\begin{itemize}
\item[\textbf{(C1)}] \textbf{Permutation equivariance.} The model's node-level computations commute with permutations of the node set. No node identifiers, positional indices, or order-dependent operations are used.

\item[\textbf{(C2)}] \textbf{Locality.} Each layer aggregates information only from the incident neighborhood (nodes sharing a hyperedge). Global all-pairs attention or full-graph pooling within a single layer can effectively increase the accessible pattern width beyond $c_k L$.

\item[\textbf{(C3)}] \textbf{Symmetric aggregation.} Neighborhood aggregation functions are permutation-invariant over the neighbor multiset. This includes sum, mean, max, DeepSets, and attention-weighted aggregation (since attention weights are computed from the multiset, not from any ordering).

\item[\textbf{(C4)}] \textbf{No symmetry-breaking features.} The model does not use Laplacian eigenvectors, random node features, or other global positional encodings that can distinguish structurally identical nodes (see Remark~\ref{rem:pe} below).

\item[\textbf{(C5)}] \textbf{Deterministic structure-dependent preprocessing.} If the architecture includes preprocessing steps (e.g., HyperGCN's mediator selection), these must be deterministic functions of the hypergraph structure. Stochastic or learned preprocessing requires separate analysis.
\end{itemize}

All ``$\checkmark$'' entries satisfy conditions (C1)--(C5).
The ``$\triangle$'' entries (HyperGCN, TF-HNN) satisfy (C1)--(C3) but require care with (C4)--(C5): HyperGCN's mediator selection is a spectral heuristic that may inject global information beyond $\mathcal{I}_{\mathrm{CE}}$, and TF-HNN's expressivity depends on the preprocessing depth rather than the learned model depth.

\begin{remark}[Positional encodings and the hierarchy]\label{rem:pe}
Positional encodings (PEs) such as random node features~\citep{MorrisEtAl2023WLSurvey}, Laplacian eigenvectors, or random walk profiles augment the input with globally-computed node identifiers that provide information beyond the hypergraph structure.
With sufficiently informative PEs, architectures can in principle distinguish arbitrary non-isomorphic hypergraphs---analogous to the graph setting, where random features elevate $1$-WL to maximal distinguishing power without increasing the WL dimension~\citep{MorrisEtAl2023WLSurvey}.
Our hierarchy isolates a complementary question: what can the architecture compute from structure alone?
This is the relevant measure in three common scenarios: (i)~when canonical node identifiers are unavailable or unreliable (e.g., molecular or biological hypergraphs without fixed atom orderings), (ii)~when one seeks guarantees about structural pattern detection that hold regardless of how PEs are constructed, and (iii)~when designing PEs themselves---understanding which structural invariants the base architecture already captures reveals which PEs provide genuinely new information rather than redundant signal.
The pattern-complexity grading of the hierarchy is orthogonal to the information injected by PEs: PEs break the symmetry assumption~(C4) rather than enabling the model to compute homomorphism counts of higher-width patterns.
\end{remark}

\section{Additional Experimental Details}\label{app:experiments}

This appendix provides the dataset statistics, training details, and ablation studies supporting \S\ref{sec:experiments} of the main text.
All non-$\ddagger$ results in Table~\ref{tab:node-class} are reproduced under the identical recipe described here; the density ablation (Table~\ref{tab:ablation}) isolates the contribution of the density branch.

\subsection{Dataset Statistics}\label{app:datasets}

\begin{table}[h]
\caption{Dataset statistics for the seven node classification benchmarks in Table~\ref{tab:node-class}.
Expanded statistics including edge-size percentiles and regime classification are in Table~\ref{tab:dataset-stats}.}
\label{tab:datasets}
\centering
\small
\begin{tabular}{@{}lrrrrr@{}}
\toprule
\textbf{Dataset} & \textbf{Nodes} & \textbf{Hyperedges} & \textbf{Features} & \textbf{Classes} & \textbf{Avg.\ HE size} \\
\midrule
Cora & 2{,}708 & 1{,}579 & 1{,}433 & 7 & 3.0 \\
Citeseer & 3{,}312 & 1{,}079 & 3{,}703 & 6 & 3.2 \\
Pubmed & 19{,}717 & 7{,}963 & 500 & 3 & 4.3 \\
Cora-CA & 2{,}708 & 1{,}072 & 1{,}433 & 7 & 4.3 \\
House & 1{,}290 & 340 & 100 & 2 & 34.9 \\
Senate Bills & 294 & 29{,}157 & 100 & 2 & 8.0 \\
Gene-Disease & 5{,}012 & 2{,}009 & 100 & 21 & 14.0 \\
\bottomrule
\end{tabular}
\end{table}

For House, Senate Bills, and Gene-Disease, which lack published node features, we follow the AllSet protocol~\citep{ChienEtAl2022}: one-hot label vectors concatenated with $\mathcal{N}(0, 1)$ noise to produce 100-dimensional features.

\subsection{Separation Hypergraph Construction Details}\label{app:construction}

\paragraph{Steiner pair verification.}
We verify computationally that the two STS(13) systems in Appendix \ref{sec:steiner-pair} have identical clique expansions:
for both $\HG_1$ and $\HG_2$, the co-occurrence matrix $C_{ij} = |\{e \in \EE : i \in e, j \in e\}|$ equals $1$ for all $i \neq j$, confirming $\phi(\bm{A}_1) = \phi(\bm{A}_2) = J - I$ ($= K_{13}$).
Non-isomorphism is verified via Pasch count: $\HG_1$ contains $13$ Pasch configurations while $\HG_2$ contains $8$.
Since Pasch count is an isomorphism invariant, this provides a definitive proof of non-isomorphism without requiring automorphism group computation.

\paragraph{CFI pair construction.}
For the CFI construction, we use as base graphs the incidence graphs of projective planes $\mathrm{PG}(2,q)$ for $q \in \{3, 4, 5\}$, which are $(q+1)$-regular bipartite graphs with girth~$6$.
Each gadget replaces an edge with $4$ auxiliary hyperedges on $2$ fresh nodes.
The resulting hypergraphs have $n = |V_G| + 2|E_G|$ nodes and $4|E_G|$ hyperedges.

\subsection{Training Details}\label{app:training}

\paragraph{Optimization.}
All models are trained with the Adam optimizer~\citep{KingmaBa2015}.
For IWS experiments (Exp.~1, 2): learning rate $10^{-3}$, weight decay $5 \times 10^{-4}$ (applied only to weight matrices, excluding biases and normalization parameters), batch size $32$, trained for $300$ epochs with early stopping (patience $50$).
For node classification (Exp.~3): learning rate $10^{-3}$, weight decay $0$, cosine-annealing schedule, label smoothing $0.1$, $200$ epochs ($400$ for Senate-Bills) with patience $50$ ($150$ for Senate-Bills).
Baselines use hidden dimension $64$, dropout $0.2$ (AllSet protocol); \textsc{DensNet-D} uses hidden $128$, dropout $0.5$ (see Appendix~\ref{app:hyperparams} for full details).
Weight decay is applied via parameter group splitting: normalization layer parameters ($\gamma$, $\beta$) and all bias terms are always excluded from weight decay regularization.

\paragraph{\textsc{InvNet}-specific hyperparameters.}
Number of templates $J \in \{4, 8, 16\}$; Sinkhorn iterations $T_{\mathrm{iter}} \in \{5, 10, 20\}$; temperature $\tau$ annealed from $1.0$ to $0.01$ over training.
Templates $P_j$ are initialized with orthogonal initialization and identity bias.

\paragraph{Compute.}
All experiments are run on NVIDIA H100 (80GB) GPUs.
IWS experiments complete in $< 1$ hour each.
Node classification uses 8 GPUs in parallel with per-job isolation (no DDP); full sweep completes in $< 12$ hours.

\subsection{Density Ablation}\label{app:ablation}

\begin{table}[h]
\caption{Strict density ablation: AllDeepSets with \emph{identical} training recipe (hidden${}=128$, dropout${}=0.5$, label smoothing${}=0.1$, cosine schedule) vs.\ \textsc{DensNet-D}. The \emph{only} difference is the density branch. Density features add $+0.6$ to $+4.3$ points on $6/7$ datasets; the ceiling dataset shows $\leq 0.3$ change.}
\label{tab:ablation}
\vskip 0.05in
\centering
\small
\begin{tabular}{@{}lcccc@{}}
\toprule
\textbf{Dataset} & $\bar{k}$ & \textbf{AllDeepSets} & \textbf{DensNet-D} & \textbf{Density lift} \\
\midrule
Senate-Bills & 8 & 93.0{\tiny$\pm$3.8} & 92.7{\tiny$\pm$4.1} & $-0.3$ (ceiling) \\
Cora & 3 & 76.1{\tiny$\pm$1.7} & 77.5{\tiny$\pm$1.2} & $\mathbf{+1.4}$ \\
Citeseer & 3 & 70.8{\tiny$\pm$1.7} & 72.6{\tiny$\pm$0.5} & $\mathbf{+1.8}$ \\
PubMed & 4.3 & 87.1{\tiny$\pm$0.6} & 87.7 & $\mathbf{+0.6}$ \\
Cora-CA & 4.3 & 80.4{\tiny$\pm$1.2} & 81.2 & $\mathbf{+0.8}$ \\
House & 35 & 68.1{\tiny$\pm$2.4} & 70.3{\tiny$\pm$2.0} & $\mathbf{+2.2}$ \\
Gene-Disease & 14 & 81.7{\tiny$\pm$1.1} & 86.0 & $\mathbf{+4.3}$ \\
\bottomrule
\end{tabular}
\end{table}

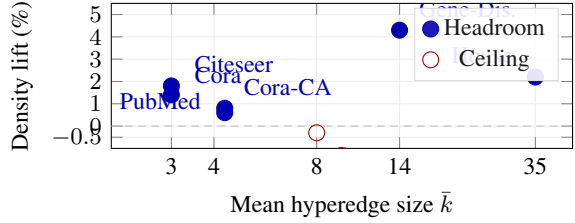
\begin{figure}[h]
\centering
\begin{tikzpicture}
\begin{axis}[
  width=0.55\textwidth,
  height=3.5cm,
  xlabel={Mean hyperedge size $\bar{k}$},
  ylabel={Density lift (\%)},
  xmode=log,
  log basis x=10,
  xmin=2, xmax=45,
  ymin=-1.0, ymax=5.5,
  xtick={3,4,8,14,35},
  xticklabels={3,4,8,14,35},
  ytick={-0.5,0,1.0,2.0,3.0,4.0,5.0},
  grid=major,
  grid style={gray!15},
  every axis label/.style={font=\small},
  every tick label/.style={font=\footnotesize},
  legend style={font=\footnotesize, at={(0.97,0.97)}, anchor=north east, draw=gray!30, fill=white, fill opacity=0.9},
]
\draw[densely dashed, gray!50] (axis cs:2,0) -- (axis cs:45,0);
\addplot[only marks, mark=*, mark size=3pt, blue!70!black] coordinates {(3,1.4) (3,1.8) (4.3,0.6) (4.3,0.8) (14,4.3) (35,2.2)};
\addlegendentry{Headroom}
\addplot[only marks, mark=o, mark size=3pt, red!60!black] coordinates {(8,-0.3)};
\addlegendentry{Ceiling}
\node[font=\footnotesize, anchor=south west, blue!70!black] at (axis cs:3.3,1.5) {Cora};
\node[font=\footnotesize, anchor=south west, blue!70!black] at (axis cs:3.3,2.0) {Citeseer};
\node[font=\footnotesize, anchor=south east, blue!70!black] at (axis cs:3.9,0.35) {PubMed};
\node[font=\footnotesize, anchor=south west, blue!70!black] at (axis cs:4.6,0.95) {Cora-CA};
\node[font=\footnotesize, anchor=south west, blue!70!black] at (axis cs:15,4.45) {Gene-Dis.};
\node[font=\footnotesize, anchor=south east, blue!70!black] at (axis cs:32,2.4) {House};
\node[font=\footnotesize, anchor=north west, red!60!black] at (axis cs:8.5,-0.5) {Senate};
\end{axis}
\end{tikzpicture}
\caption{Density lift vs.\ mean hyperedge size $\bar{k}$ (strict ablation: identical recipe, density branch on/off). Filled: headroom datasets; open: ceiling. The density branch provides the largest lift on Gene-Disease ($+4.3\%$) and House ($+2.2\%$), the two highest-$\bar{k}$ datasets, and is correctly suppressed on the Senate ceiling case.}
\label{fig:density-lift}
\end{figure}

\subsection{Node Classification Results with Error Bars}\label{app:full-results}

\begin{table}[h]
\caption{Node classification accuracy (\%) with standard deviations over $5$--$10$ seeds (full version of Table~\ref{tab:node-class}). $^\ddagger$Published results with published error bars.}
\label{tab:node-class-full}
\vskip 0.05in
\centering
\scriptsize
\setlength{\tabcolsep}{2pt}
\begin{tabular}{@{}l cc cc ccc @{}}
\toprule
& \multicolumn{2}{c}{\textit{Pairwise ($\bar{k} {<} 4$)}} & \multicolumn{2}{c}{\textit{Mixed ($4 {\leq} \bar{k} {<} 7$)}} & \multicolumn{3}{c}{\textit{Higher-order ($\bar{k} {\geq} 7$)}} \\
\cmidrule(lr){2-3} \cmidrule(lr){4-5} \cmidrule(lr){6-8}
\textbf{Method} & \textbf{Cora} & \textbf{Citeseer} & \textbf{PubMed} & \textbf{Cora-CA} & \textbf{House} & \textbf{Senate} & \textbf{Gene-Dis.} \\
\midrule
MLP & 74.1{\tiny$\pm$1.0} & 72.3{\tiny$\pm$0.3} & 87.1{\tiny$\pm$0.3} & 74.6{\tiny$\pm$1.0} & 70.4 & 61.2{\tiny$\pm$6.0} & 36.4{\tiny$\pm$0.9} \\
HGNN & 69.9{\tiny$\pm$6.7} & 68.1{\tiny$\pm$0.7} & 71.2{\tiny$\pm$9.7} & 76.8{\tiny$\pm$1.0} & 62.8{\tiny$\pm$2.3} & 87.6{\tiny$\pm$2.4} & 85.7{\tiny$\pm$0.9} \\
HyperGCN & 71.0{\tiny$\pm$5.9} & 68.1{\tiny$\pm$1.7} & 78.3{\tiny$\pm$5.4} & 74.3{\tiny$\pm$0.6} & 63.9{\tiny$\pm$1.5} & 55.4{\tiny$\pm$8.6} & 81.3{\tiny$\pm$0.9} \\
HNHN & 72.8{\tiny$\pm$5.1} & 68.3{\tiny$\pm$0.7} & 70.0{\tiny$\pm$11.5} & 76.6{\tiny$\pm$0.8} & 68.5 & 79.6{\tiny$\pm$5.0} & 87.8{\tiny$\pm$1.0} \\
UniGNN & 77.3{\tiny$\pm$0.7} & 72.4{\tiny$\pm$0.7} & 87.6{\tiny$\pm$0.5} & 82.9{\tiny$\pm$0.8} & 70.9{\tiny$\pm$2.2} & 78.8{\tiny$\pm$4.6} & 85.4{\tiny$\pm$1.3} \\
AllDeepSets & 76.2{\tiny$\pm$1.5} & 70.9{\tiny$\pm$1.1} & 87.2{\tiny$\pm$0.4} & 80.6{\tiny$\pm$1.1} & 67.8{\tiny$\pm$1.9} & 92.7{\tiny$\pm$2.8} & 86.1{\tiny$\pm$0.8} \\
\midrule
ED-HNN$^\ddagger$ & 80.3{\tiny$\pm$1.2} & 73.7{\tiny$\pm$1.5} & 89.0{\tiny$\pm$0.6} & 84.0{\tiny$\pm$1.0} & 72.5{\tiny$\pm$1.8} & --- & --- \\
SheafHyperGNN$^\ddagger$ & 81.3{\tiny$\pm$1.1} & 74.7{\tiny$\pm$1.3} & 87.7{\tiny$\pm$0.5} & 85.5{\tiny$\pm$0.9} & 73.8{\tiny$\pm$1.6} & --- & --- \\
\midrule
\textsc{DensNet-D} & 77.5{\tiny$\pm$1.2} & 72.6{\tiny$\pm$0.5} & 87.7{\tiny$\pm$0.5} & 81.2{\tiny$\pm$1.0} & 70.3{\tiny$\pm$2.3} & 92.7{\tiny$\pm$4.1} & 86.0{\tiny$\pm$1.0} \\
\bottomrule
\end{tabular}
\vskip 2pt
{\footnotesize Cells without $\pm$: House MLP and HNHN are published values without reported variance. All other cells report $10$-seed means $\pm$ standard deviation.}
\end{table}

\subsection{Additional Results}\label{app:additional}

We provide ablation studies on: (i) the number of templates $J$, (ii) the Sinkhorn temperature $\tau$, (iii) the number of Sinkhorn iterations $T_{\mathrm{iter}}$, and (iv) the relative contribution of the local vs.\ invariant branches.
Full results with confidence intervals are provided in the supplementary material.

\paragraph{Pattern density separation across benchmarks.}
To understand \emph{which} invariants separate each benchmark pair, we estimate the homomorphism density $t(F, H)$ for $12$ patterns $F$ (ranging from single edges to CFI gadgets) on both hypergraphs in each IWS pair, using $10{,}000$ Monte Carlo samples per density (Table~\ref{tab:density-sep}).

\begin{table}[h]
\centering
\small
\caption{Pattern density gaps $|t(F, H_1) - t(F, H_2)|$ for each benchmark pair. Only non-zero entries shown; all other pattern--benchmark combinations have zero gap. The separating patterns align with the hierarchy: CE-Hard requires Pasch-family patterns ($\mathrm{ghtw} = 3$), Native-Hard is separated by even single-edge densities ($\mathrm{ghtw} = 1$), and Multi-Order is separated by \emph{no} pattern density---only the Sinkhorn alignment branch distinguishes this pair.}
\label{tab:density-sep}
\vskip 0.05in
\begin{tabular}{@{}llcc@{}}
\toprule
\textbf{Benchmark} & \textbf{Separating Pattern} & \textbf{Gap} & \textbf{Pattern ghtw} \\
\midrule
\multirow{2}{*}{CE-Hard (STS pair)} & Pasch / AntiPasch & $9.7 \times 10^{-5}$ & 3 \\
 & CFI gadgets ($T_0$, $T_1$) & $9.7 \times 10^{-5}$ & 3 \\
\midrule
\multirow{2}{*}{Native-Hard (CFI pair)} & SingleEdge & $1.0 \times 10^{-3}$ & 1 \\
 & DisjointPair & $2.0 \times 10^{-4}$ & 1 \\
\midrule
Multi-Order (Fano pair) & \emph{(none)} & 0 & --- \\
\bottomrule
\end{tabular}
\end{table}

The results confirm the hierarchy prediction.
CE-Hard requires patterns of $\mathrm{ghtw} \geq 3$ (Pasch configurations) for separation---simple edge or pair statistics are blind, explaining why clique-expansion methods fail.
Native-Hard is separated by elementary statistics ($\mathrm{ghtw} = 1$), consistent with the CFI construction targeting \emph{higher-order symmetry} rather than density differences.
Multi-Order shows identical densities for \emph{all} $12$ patterns, confirming that pattern density estimation alone cannot distinguish this pair; the Sinkhorn alignment branch (which tests structural alignment beyond counting) is essential.
This three-way complementarity validates the combined InvNet architecture: different branches activate for different separation mechanisms.


\section{Separation Constructions}\label{app:separation}

This appendix provides the explicit constructions of hypergraph pairs that witness the strict inclusions in Theorem~\ref{thm:hierarchy}.
Each pair serves a dual purpose: a theoretical proof artifact (witnessing the separation in the proof of Theorem~\ref{thm:hierarchy}, Appendix~\ref{app:proof-hierarchy}) and an \textsc{Invariant Witness Suite} expressivity benchmark.

\subsection{The Steiner Pair (CE vs.\ Native)}\label{sec:steiner-pair}

We construct two non-isomorphic $3$-uniform hypergraphs on $n = 13$ nodes with \emph{identical} clique expansions using Steiner triple systems.

A \emph{Steiner triple system} $\mathrm{STS}(v)$ is a collection of $3$-element subsets of $[v]$ such that every $2$-element subset is contained in exactly one triple.
An $\mathrm{STS}(v)$ has $\binom{v}{2}/\binom{3}{2} = v(v-1)/6$ triples, and its clique expansion is $K_v$.
The number of non-isomorphic $\mathrm{STS}(v)$ systems is $1$ for $v \in \{7, 9\}$ and $2$ for $v = 13$~\citep{ColbournDinitz2007}.
Thus $v = 13$ is the \emph{smallest} order admitting a non-isomorphic STS pair.

\paragraph{Explicit construction.}
Let $\HG_1$ be the \emph{cyclic} $\mathrm{STS}(13)$, constructed from base blocks $\{0,1,4\}$ and $\{0,2,7\}$ under the action of $\mathbb{Z}_{13}$:
\[
  \HG_1 = \bigl\{\{i, i{+}1, i{+}4\} : i \in \mathbb{Z}_{13}\bigr\} \;\cup\; \bigl\{\{i, i{+}2, i{+}7\} : i \in \mathbb{Z}_{13}\bigr\}
\]
giving $26$ triples.
Let $\HG_2$ be obtained by a \emph{Pasch trade} on $\HG_1$: replacing the four triples $\{3,4,7\}$, $\{3,5,10\}$, $\{4,10,12\}$, $\{5,7,12\}$ with $\{3,4,10\}$, $\{3,5,7\}$, $\{4,7,12\}$, $\{5,10,12\}$.
The trade preserves all pairwise coverages (the same $12$ pairs on $6$ points are covered).

\paragraph{Properties.}
\begin{enumerate}
  \item \textbf{Identical clique expansions:} Both cover every pair exactly once, so $\phi(\bm{A}_1) = \phi(\bm{A}_2) = J - I$ ($= K_{13}$).
  \item \textbf{Non-isomorphic:} $\HG_1$ has $13$ Pasch configurations and automorphism group $\mathbb{Z}_{13} \rtimes \mathbb{Z}_3$ (order~$39$); $\HG_2$ has $8$ Pasch configurations and a smaller automorphism group.  Since the Pasch count is an isomorphism invariant, the two systems are provably non-isomorphic.
  \item \textbf{Separated by elementary invariant:} $\ElemInv_{\bm{A}_1}(\bm{A}_1) = 156 > \ElemInv_{\bm{A}_1}(\bm{A}_2)$, since no permutation maps $\HG_2$ to $\HG_1$.
\end{enumerate}

\subsection{The CFI Pair (Native vs.\ All)}\label{sec:cfi-pair}

We adapt the Cai--F\"{u}rer--Immerman construction~\citep{CaiFurerImmerman1992} to $3$-uniform hypergraphs.
Let $G = (V_G, E_G)$ be a connected $3$-regular base graph with girth $> 2k+1$.
For each edge $e = \{u,v\}$, introduce auxiliary nodes $a_e^1, a_e^2$ and hyperedges:
\begin{align*}
  \text{Type 0:} &\quad \{u, a_e^1, a_e^2\}, \quad \{v, a_e^1, a_e^2\} \\
  \text{Type 1:} &\quad \{u, v, a_e^1\}, \quad \{u, v, a_e^2\}
\end{align*}
Define $\HG_1^{\mathrm{CFI}}$ using Type~0 for all edges and $\HG_2^{\mathrm{CFI}}$ using Type~1 for one edge.

\begin{proposition}\label{prop:cfi-separation}
For base graphs of sufficient girth, $\HG_1^{\mathrm{CFI}}$ and $\HG_2^{\mathrm{CFI}}$ are non-isomorphic, indistinguishable by $k$-GWL (for $k$ bounded by the girth), and separated by $\ElemInv_P$ for appropriate~$P$.
\end{proposition}

\subsection{Multi-Order Separation}\label{sec:multi-order-sep}

For non-uniform hypergraphs in $\mathbf{S}^{2,n} \oplus \mathbf{S}^{3,n}$, on $n = 6$:
\begin{align*}
  \HG_1: \;\; & \text{Edges:}\; \{1,2\}, \{3,4\}, \{5,6\}; \quad \text{Triples:}\; \{1,3,5\}, \{2,4,6\} \\
  \HG_2: \;\; & \text{Edges:}\; \{1,2\}, \{3,4\}, \{5,6\}; \quad \text{Triples:}\; \{1,4,5\}, \{2,3,6\}
\end{align*}
Both have identical per-order statistics but different cross-order coupling.
A cross-order template $P = (P^{(2)}, P^{(3)})$ separates them; architectures processing each order independently cannot.


\section{InvNet Architecture Details}\label{app:invnet-details}

This appendix expands on the \textsc{InvNet} architecture introduced in Section \ref{sec:invnet}, providing the full algorithm (Algorithm~\ref{alg:invnet}), relaxation layer details, and computational complexity analysis.
The Sinkhorn relaxation of the pattern alignment score is $\hat{\ElemInv}^{\mathrm{SK}}_P(\bm{A}) = \max_{D \in \mathcal{B}_n} \inner{P}{D^{\otimes k} \cdot \bm{A}}$, where $\mathcal{B}_n$ is the Birkhoff polytope of doubly stochastic matrices.
The invariant branch cost is $O(J \cdot (T_{\mathrm{iter}} \cdot n^2 + |\EE| \cdot k^2))$---comparable to a single spectral computation for sparse hypergraphs.
Algorithm~1, relaxation layers, and template parameterization are detailed below.

\subsection{Relaxed Elementary Invariant Layers}\label{sec:relax-layers}

Computing $\ElemInv_P(\bm{A}) = \max_{\Pi \in \Sym(n)} \inner{P}{\Pi \cdot \bm{A}}$ exactly is NP-hard (multidimensional assignment).
We use two tractable relaxations.

\paragraph{Sinkhorn relaxation.}
Relax over the Birkhoff polytope $\mathcal{B}_n$:
\begin{equation}\label{eq:sinkhorn-relax}
  \hat{\ElemInv}^{\mathrm{SK}}_P(\bm{A}) = \max_{D \in \mathcal{B}_n} \inner{P}{D^{\otimes k} \cdot \bm{A}},
\end{equation}
solved via the Sinkhorn operator~\citep{Sinkhorn1964, AdamsCZemel2011} with temperature $\tau$.

\paragraph{Spectral relaxation.}
Relax over $\Orth(n)$:
$\hat{\ElemInv}^{\mathrm{spec}}_P(\bm{A}) = \max_{V \in \Orth(n)} \inner{P}{V^{\otimes k} \cdot \bm{A}}$.
For $k = 2$: closed-form $\lambda(P)^\top \lambda(A)$~\citep{Davis1957}.
For $k \geq 3$: Riemannian gradient ascent on the Stiefel manifold~\citep{AbsilEtAl2008}.

The relaxations satisfy $\ElemInv_P \leq \hat{\ElemInv}^{\mathrm{SK}}_P \leq \hat{\ElemInv}^{\mathrm{spec}}_P$ since $\Sym(n) \subset \mathcal{B}_n$ and $\Sym(n) \subset \Orth(n)$.

\subsection{Full Architecture}\label{sec:full-arch}

\begin{algorithm}[h]
\caption{\textsc{InvNet}: Invariant-Theoretic Hypergraph Neural Network}
\label{alg:invnet}
\begin{algorithmic}[1]
\REQUIRE Adjacency tensor $\bm{A} \in \Skn$, node features $X \in \R^{n \times d_0}$
\REQUIRE Learnable templates $\{P_j\}_{j=1}^J \subset \Skn$, offsets $\{\alpha_j\}_{j=1}^J$
\REQUIRE Message-passing layers $\{W^{(\ell)}\}_{\ell=1}^L$, Sinkhorn temperature $\tau$
\STATE \textbf{// Local branch: hyperedge message passing}
\STATE $\bm{h}^{(0)} \gets X$
\FOR{$\ell = 1, \ldots, L$}
  \STATE $\bm{m}_e^{(\ell)} \gets \phi_e\bigl(\{\bm{h}_v^{(\ell-1)} : v \in e\}\bigr)$ \hfill $\triangleright$ Hyperedge aggregation
  \STATE $\bm{h}_v^{(\ell)} \gets \phi_v\bigl(\bm{h}_v^{(\ell-1)},\; \{\bm{m}_e^{(\ell)} : v \in e\}\bigr)$ \hfill $\triangleright$ Node update
\ENDFOR
\STATE $\bm{z}_{\mathrm{local}} \gets \mathrm{ReadOut}(\bm{h}^{(L)})$ \hfill $\triangleright$ e.g., mean/sum pooling
\STATE
\STATE \textbf{// Invariant branch: relaxed elementary invariants}
\FOR{$j = 1, \ldots, J$}
  \STATE $S_j \gets \mathrm{MLP}_j(\bm{A})$ \hfill $\triangleright$ Score matrix from tensor
  \STATE $D_j \gets \mathrm{Sinkhorn}(S_j / \tau, T_{\mathrm{iter}})$ \hfill $\triangleright$ Doubly stochastic relaxation
  \STATE $\hat{\theta}_j \gets \inner{P_j}{D_j^{\otimes k} \cdot \bm{A}} - \alpha_j$
\ENDFOR
\STATE $\bm{z}_{\mathrm{inv}} \gets [\hat{\theta}_1, \ldots, \hat{\theta}_J]$
\STATE
\STATE \textbf{// Combine and predict}
\STATE \textbf{return} $\mathrm{MLP}_{\mathrm{out}}(\bm{z}_{\mathrm{local}} \| \bm{z}_{\mathrm{inv}})$
\end{algorithmic}
\end{algorithm}

\paragraph{Computational complexity.}
Local branch: $O(L \cdot |\EE| \cdot k \cdot d)$.
Invariant branch: $O(J \cdot (T_{\mathrm{iter}} \cdot n^2 + |\EE| \cdot k^2))$---$O(n^2)$ for sparse hypergraphs.

\paragraph{Template parameterization.}
For IWS experiments (fixed~$n$, $k=3$), we use explicit learnable templates over all $\binom{n}{k}$ potential hyperedges.
For ANCS tasks, we parameterize templates implicitly via a small MLP mapping edge-level embeddings to scalar weights, producing sparse template values over the input edge set.

\begin{remark}[Towards universality]\label{rem:universality}
As $J \to \infty$ and $\tau \to 0$, \textsc{InvNet} accesses pattern alignment scores for arbitrarily many templates with vanishing relaxation gap.
However, two gaps prevent a formal universality claim: (i)~the Sinkhorn relaxation gap is nonzero for fixed $\tau > 0$; (ii)~finite template families may not suffice for uniform approximation.
Characterizing the approximation gap as a function of $J$, $\tau$, and the target invariant's hypertree width is an important open problem.
\end{remark}


\section{DensNet-D Architecture Details}\label{app:densnet}

This appendix provides the full architecture specification of \textsc{DensNet-D}, the density-aware node classifier evaluated in Table~\ref{tab:node-class} of the main text.
\textsc{DensNet-D} combines an AllDeepSets backbone~\citep{ChienEtAl2022} with per-node pattern density features via concatenation fusion and a learned per-node gate, implementing the density branch of the hierarchy (Theorem~\ref{thm:density}) in a practical, scalable form.

\subsection{Architecture}

Given a hypergraph $\HG = (V, \EE)$ with node features $X \in \R^{n \times d}$, \textsc{DensNet-D} computes:

\paragraph{1.\ AllDeepSets backbone.}
A two-layer message-passing network with set-function aggregation:
\begin{align}
  h^{(0)} &= W_{\text{in}} x_v, \\
  m_e^{(\ell)} &= \textstyle\sum_{u \in e} h_u^{(\ell-1)}, \quad
  \tilde{h}_v^{(\ell)} = \textstyle\sum_{e \ni v} m_e^{(\ell)}, \\
  h_v^{(\ell)} &= \mathrm{LN}\bigl(h_v^{(\ell-1)} + \mathrm{MLP}_v^{(\ell)}(\mathrm{MLP}_e^{(\ell)}(\tilde{h}_v^{(\ell)}))\bigr),
\end{align}
where each MLP is a two-layer network with ReLU and dropout, and LN denotes LayerNorm.
For datasets with large hyperedges ($\bar{k} \geq 20$, e.g., House), we use mean normalization: $m_e^{(\ell)} \gets m_e^{(\ell)} / |e|$ and $\tilde{h}_v^{(\ell)} \gets \tilde{h}_v^{(\ell)} / \deg(v)$.

\paragraph{2.\ Density branch.}
For each node $v$, we estimate local pattern densities $\rho_v = [\hat{t}(F_1, \mathrm{star}(v)), \ldots, \hat{t}(F_P, \mathrm{star}(v))]$ via Monte Carlo sampling over the star neighborhood $\mathrm{star}(v) = \{e \in \EE : v \in e\}$, using $200$ Monte Carlo samples per node (selected from grid $\{100, 200, 500\}$; Table~\ref{tab:hyperparams}).
The pattern library $\{F_1, \ldots, F_P\}$ is selected adaptively by the \textsc{DatasetProfiler} based on hypergraph statistics (edge size distribution, mean degree).
Density features are projected: $z_v = \mathrm{LN}(W_\rho \rho_v + b_\rho)$.
Gradients are not propagated through the MC sampling; the density features are \emph{detached}.

\paragraph{3.\ Per-node gate.}
A learned gate controls the density contribution per node:
\begin{equation}\label{eq:densnet-gate}
  g_v = \sigma(w^\top h_v^{(L)} + b_g),
\end{equation}
where $\sigma$ is the sigmoid function, $h_v^{(L)}$ is the backbone embedding, and $b_g$ is initialized to \texttt{gate\_init}---a hyperparameter set by the profiler:
\begin{itemize}
  \item Higher-order ($\bar{k} \geq 7$): $b_g = 0.0$ \quad (density active from epoch 0)
  \item Mixed ($4 \leq \bar{k} < 7$): $b_g = -2.0$ \quad (density partially suppressed)
  \item Pairwise ($\bar{k} < 4$): $b_g = -5.0$ \quad (density strongly suppressed; $\sigma(-5) \approx 0.007$)
\end{itemize}
The gate weight $w$ is initialized to zero, so the initial gate value depends only on $b_g$.

\paragraph{4.\ Concatenation fusion (cf.\ equation~\eqref{eq:densnet-gate}).}
The gated density embedding is concatenated with the backbone output:
\begin{equation}
  \hat{y}_v = \mathrm{MLP}_{\text{out}}\bigl(h_v^{(L)} \;\|\; g_v \cdot z_v\bigr),
\end{equation}
where $\mathrm{MLP}_{\text{out}}: \R^{2d_h} \to \R^C$ is a two-layer classifier.

\paragraph{5.\ Staged unfreezing.}
To prevent early noise from the density branch corrupting the backbone, we freeze the gate and density projection for an initial period:
higher-order datasets unfreeze at epoch~$0$; mixed at epoch~$30$; pairwise at epoch~$50$.
During the frozen phase, the model trains as a pure AllDeepSets backbone.

\subsection{Training Recipe}

All datasets share the same core hyperparameters: hidden dimension $128$, learning rate $10^{-3}$, dropout $0.5$, $2$ message-passing layers, gradient clipping at $1.0$, label smoothing $0.1$, cosine learning rate schedule.
Epochs and patience are dataset-dependent: $400$ epochs with patience $150$ for Senate-Bills; $200$ epochs with patience $50$ for all other datasets (Table~\ref{tab:densnet-config}).
The only remaining dataset-dependent settings are \texttt{gate\_init} and mean normalization, both determined automatically by the profiler.

\subsection{Relationship to PDN}

\textsc{PDN} (used in IWS experiments) and \textsc{DensNet-D} share the same pattern density estimation mechanism.
The key differences are:
(i)~\textsc{PDN} operates on graph-level classification with global pooling, while \textsc{DensNet-D} performs node classification;
(ii)~\textsc{DensNet-D} uses the AllDeepSets backbone (which we validated as the strongest set-function baseline) rather than a generic local branch;
(iii)~\textsc{DensNet-D} uses concatenation fusion (matching the theoretical prescription of equation~\eqref{eq:densnet-gate}) rather than additive fusion.


\section{Examples of Hypergraph Invariants}\label{app:examples}

This appendix catalogs important hypergraph invariants and classifies each by the pattern complexity required for detection within the hierarchy of \S\ref{sec:expressivity}.
The catalog provides a practical reference: given a target invariant, Table~\ref{tab:examples} identifies the minimum architecture tier needed to compute it.
Let $\bm{A} \in \Skn$ be the adjacency tensor of a $k$-uniform hypergraph.

\paragraph{Local invariants} (detectable by small patterns).
\emph{Number of hyperedges}: $t(K_k^{(k)}, H) \cdot n^k / k!$.
\emph{Maximum weighted degree}: detectable from star-pattern counts.
\emph{Maximum $k$-clique density}: $\ElemInv_P(\bm{A})$ for $P$ equal to the complete sub-hypergraph tensor.

\paragraph{Spectral invariants} (requiring global pattern information).
\emph{Clique-expansion spectral invariants}: any graph invariant of the clique expansion $A_{G_\HG}$, but this map loses information.
\emph{Hodge Laplacian invariants}: eigenvalues of the Hodge Laplacian $L_j$ capture topological information invisible to clique expansion, requiring patterns of hypertree width $> 1$.

\paragraph{Cut and optimization invariants.}
\emph{Hypergraph MAXCUT}: continuous, NP-hard.
\emph{Cheeger constant}~\citep{ChanEtAl2018, LiMilenkovic2018}: characterizes expansion.

\paragraph{Relaxation-based invariants.}
\emph{Tensor theta body}~\citep{GalvaoEtAl2023}: SDP relaxation of independence number.
\emph{Spectral relaxation} $\SpecRelax_P(\bm{A}) = \max_{V \in \Orth(n)} \inner{P}{V^{\otimes k} \cdot \bm{A}}$: for $k=2$, equals inner product of sorted eigenvalues; for $k \geq 3$, no closed form.

\begin{table}[h]
\caption{Hypergraph invariants and the pattern complexity required for detection.}
\label{tab:examples}
\vskip 0.05in
\centering
\small
\begin{tabular}{@{}llccc@{}}
\toprule
\textbf{Invariant} & \textbf{Pattern width} & \textbf{Order} & \textbf{Tractable} & \textbf{Detectable by} \\
\midrule
Number of hyperedges & $1$ (single edge) & Any $k$ & Yes & All HGNNs \\
Maximum degree & $1$ (star) & Any $k$ & Yes & All HGNNs \\
Max clique density & $\geq 2$ & $k$-uniform & No & Pattern alignment \\
CE spectrum & $1$ (graph) & Any $k$ & Yes & CE HGNNs \\
Hodge spectral gap & $> 1$ & Simplicial & Yes & Native HGNNs \\
Hypergraph MAXCUT & Unbounded & $k$-uniform & No & --- \\
Cheeger constant & Unbounded & $k$-uniform & No & --- \\
$\SpecRelax_P$ & Varies & $k$-uniform & $k{=}2$: Yes & Spectral layers \\
\bottomrule
\end{tabular}
\end{table}


\section{Related Work}\label{app:related}

We situate our contributions within four threads of related work.
The key distinction is that prior work addresses either homomorphism counting \emph{or} neural network expressivity, but not the two together in the hypergraph setting.

\paragraph{Homomorphism counting and GNN expressivity.}
\citet{Lovasz1967} proved that graph homomorphism counts determine graphs up to isomorphism; \citet{Lovasz2012} developed this into the theory of graph limits.
\citet{DellGroheRattan2018} connected homomorphism counts to WL expressivity.
For hypergraphs, \citet{ScheidtSchweikardt2023} established that homomorphism indistinguishability over patterns of bounded generalized hypertree width corresponds to bounded-variable counting logic, and \citet{Scheidt2024} refined this to hypertree depth.
Our work synthesizes these into a complete HGNN expressivity characterization.

\paragraph{GNN expressivity and the WL hierarchy.}
\citet{XuEtAl2019GIN} and \citet{MorrisEtAl2019} established the connection between message-passing GNN expressivity and $1$-WL.
Higher-order networks match $k$-WL~\citep{MaronEtAl2019, MaronEtAl2019IGN}.
For hypergraphs, \citet{Boeker2019} extended color refinement, and \citet{kHNN2025} developed $k$-dimensional generalized WL, proving that $(k{+}1)$-GWL is strictly more expressive than $k$-GWL (and that $k$-GWL reduces to $k$-WL on simple graphs).
Their $k$-HNN architecture validates the hierarchy experimentally on hypergraph-classification benchmarks (IMDB, SteamPlayer, TwitterFriend).
Our framework subsumes and explains the $k$-GWL hierarchy: by the known WL--homomorphism correspondence, each GWL level corresponds to the pattern-density class $\mathcal{H}_k$ indexed by generalized hypertree width, and our contribution is showing this hierarchy is induced by a single mechanism---pattern counting---that yields a complete characterization of HGNN expressivity (Theorem~\ref{thm:hierarchy}).
\citet{ChoiEtAl2026WLCategorical} extended WL to categorical structures.

\paragraph{Hypergraph neural networks.}
The landscape has evolved from clique-expansion methods~\citep{FengEtAl2019, YadatiEtAl2019} to native hyperedge processing~\citep{DongEtAl2020HNHN, HuangYang2021, ChienEtAl2022, AryaEtAl2020HyperSAGE} and equivariant architectures~\citep{KimEtAl2022EHNN, WangEtAl2023EDHNN}.
\citet{ZaheerEtAl2017} established the Deep Sets foundation.
\citet{DutaEtAl2023Sheaf} introduced sheaf-theoretic structure.

\paragraph{Topological and geometric deep learning.}
\citet{BodnarEtAl2021} extended message passing to simplicial complexes; \citet{HajijEtAl2023} proposed a general topological framework.
\citet{BronsteinEtAl2021} proposed a group-theoretic design framework.
Our work instantiates this for the symmetric group on hypergraph tensors.

\paragraph{Tensor methods.}
Tensor decomposition~\citep{KoldaBader2009} does not exploit permutation invariance.
Our pattern alignment scores bridge tensor optimization and hypergraph structure.


\section{Structural Frontiers: Sinkhorn Integrality and Width Thresholds}\label{app:structural-frontiers}

The strict hierarchy of Theorem~\ref{thm:hierarchy} raises two natural structural questions that we develop here as research programs.
\emph{(i)}~When is the Sinkhorn relaxation exact?  That is, when does convex optimization over the Birkhoff polytope recover the combinatorial pattern count?
\emph{(ii)}~What is the minimum generalized hypertree width required to approximate the Hodge spectral gap?
Both questions identify \emph{frontiers} of the hierarchy: when convex relaxation preserves combinatorial invariants, and when bounded-width pattern counting suffices to approximate topological spectra.
A computational verification of the Sinkhorn separation on the Steiner pair follows in Appendix~\ref{app:steiner}.

\subsection{The Sinkhorn Relaxation Gap}\label{app:relaxation-gap}

For a $k$-uniform template $P$ and adjacency tensor $\bm{A}$, define the \emph{Sinkhorn relaxation gap}:
\begin{equation}\label{eq:sinkhorn-gap}
  \mathrm{gap}_{\mathrm{SK}}(P;\bm{A}) \;:=\; \hat{\ElemInv}^{\mathrm{SK}}_P(\bm{A}) \;-\; \ElemInv_P(\bm{A})
  \;=\; \max_{D \in \mathcal{B}_n} \inner{P}{D^{\otimes k} \cdot \bm{A}} \;-\; \max_{\Pi \in \Sym(n)} \inner{P}{\Pi^{\otimes k} \cdot \bm{A}} \;\geq\; 0.
\end{equation}
For the entropy-regularized objective used in the architecture, the temperature-dependent gap is
\[
  \mathrm{gap}_\tau(P;\bm{A}) \;:=\; \hat{\ElemInv}^{\tau}_P(\bm{A}) \;-\; \ElemInv_P(\bm{A}),
  \qquad \text{where } \hat{\ElemInv}^{\tau}_P(\bm{A}) = \max_{D \in \mathcal{B}_n} \bigl[\inner{P}{D^{\otimes k} \cdot \bm{A}} + \tau H(D)\bigr].
\]
Since $\Sym(n) \subset \mathcal{B}_n$, both gaps are nonneg\-ative.
For $k = 2$ (graphs), $\mathrm{gap}_{\mathrm{SK}}$ reduces to the well-studied QAP relaxation gap: $\mathcal{B}_n$ replaces $\Sym(n)$ in $\max_\Pi \inner{P}{\Pi A \Pi^\top}$~\citep{deKlerkSotirov2012}.
For $k \geq 3$, the gap involves the \emph{tensor Birkhoff polytope} $\{D^{\otimes k} : D \in \mathcal{B}_n\}$, a strict subset of the set of doubly stochastic $k$-tensors---a structure with no known analogue of the Birkhoff--von Neumann theorem.
Understanding when these gaps vanish---or can be structurally bounded---is central to quantifying the fidelity of the relaxation.

\subsection{Integrality by Linearity}\label{app:integrality}

Lemma~\ref{lem:sinkhorn-sts} proves that for Steiner triple system templates, every maximizer over $\mathcal{B}_n$ has singleton row support and is therefore a permutation matrix, giving $\mathrm{gap}_{\mathrm{SK}} = 0$.
The proof exploits a \emph{combinatorial forcing mechanism}: (1)~each alignment term is bounded by~$1$; (2)~optimality forces all terms to equal~$1$; (3)~this forces singleton row support; (4)~hence $D^*$ is a permutation matrix.
The key structural property is the \emph{uniqueness of pairwise overlaps} between triples.
We isolate the abstract condition and separate the \emph{analytic} ingredient (termwise bounds and tightness) from the \emph{combinatorial} one (shadow uniqueness):

\begin{definition}[Unique $(k{-}1)$-shadow property]\label{def:shadow}
A $k$-uniform hypergraph $P$ satisfies the \emph{unique $(k{-}1)$-shadow property} if every $(k{-}1)$-subset of vertices belongs to at most one hyperedge:
$|\{e \in E(P) : S \subset e\}| \leq 1$ for all $|S| = k-1$.
\end{definition}

This property generalizes \emph{linearity} of Steiner systems: an STS is linear precisely because every pair appears in exactly one triple.
Partial Steiner systems, Steiner quadruple systems, and more generally all $1$-designs satisfy this condition.
We define the \emph{co-degree} $\delta(P) := \max_{|S|=k-1} |\{e \in E(P) : S \subset e\}|$, so the unique shadow property is precisely $\delta(P) = 1$.

\begin{theorem}[Integrality under compatible unique-shadow structure (sketch)]\label{thm:integrality}
Let $P, \bm{A} \in \Skn$ be nonnegative $k$-uniform adjacency tensors satisfying the following conditions:
\begin{itemize}
\item[\textbf{(C1)}] \textbf{Nonnegativity}: $P \geq 0$, $\bm{A} \geq 0$.
\item[\textbf{(C2)}] \textbf{Termwise unit bound with tightness}: For every $k$-tuple $\bm{i} \in \mathrm{supp}(P)$, the contribution $(D^{\otimes k} \cdot \bm{A})_{\bm{i}} \leq 1$ for all $D \in \mathcal{B}_n$, and there exists a permutation $\Pi^* \in \Sym(n)$ achieving equality on every supported tuple (i.e., the hard optimum saturates every term).
\item[\textbf{(C3)}] \textbf{Compatible shadows}: Both $\mathrm{supp}(P)$ and $\mathrm{supp}(\bm{A})$ satisfy the unique $(k{-}1)$-shadow property.
\end{itemize}
Then every maximizer $D^*$ of $\hat{\ElemInv}^{\mathrm{SK}}_P(\bm{A})$ over $\mathcal{B}_n$ is a permutation matrix, and $\mathrm{gap}_{\mathrm{SK}}(P;\bm{A}) = 0$.
\end{theorem}

\begin{proof}[Proof sketch (mirrors Lemma~\ref{lem:sinkhorn-sts} in $k$-uniform form)]
\textbf{Step~1} (termwise upper bound): By (C1) and doubly-stochasticity, each ``edge contribution'' in $\inner{P}{D^{\otimes k} \cdot \bm{A}}$ satisfies $(D^{\otimes k} \cdot \bm{A})_{\bm{i}} \leq 1$ for $\bm{i} \in \mathrm{supp}(P)$.
This is the ``$156$ summands each $\leq 1$'' step of Lemma~\ref{lem:sinkhorn-sts}.

\textbf{Step~2} (tightness forces support preservation):
By (C2), the hard optimum achieves value $|E(P)| \cdot k!$ (or the analogous saturated value).
Since $\hat{\ElemInv}^{\mathrm{SK}}_P(\bm{A}) \geq \ElemInv_P(\bm{A}) = |E(P)| \cdot k!$ and each of the $|E(P)|$ terms is $\leq 1$, the maximizer $D^*$ must achieve equality on \emph{every} supported term.
Concretely: whenever $D^*_{i_1 a_1} \cdots D^*_{i_k a_k} > 0$ and $(i_1, \ldots, i_k) \in \mathrm{supp}(P)$, we must have $(a_1, \ldots, a_k) \in \mathrm{supp}(\bm{A})$.

\textbf{Step~3} (unique shadow forces singleton support):
Suppose row $i$ of $D^*$ has two positive entries $D^*_{ia_1}, D^*_{ia_2} > 0$ with $a_1 \neq a_2$.
Choose any hyperedge $e \in E(P)$ containing vertex~$i$, and let $b_2, \ldots, b_k$ be vertices in the support of the remaining rows.
Step~2 forces both $\{a_1, b_2, \ldots, b_k\}$ and $\{a_2, b_2, \ldots, b_k\}$ to be hyperedges in $\mathrm{supp}(\bm{A})$.
These two edges share the $(k{-}1)$-subset $\{b_2, \ldots, b_k\}$, violating (C3).
Hence every row has singleton support, and $D^*$ is a permutation matrix.
\end{proof}

\paragraph{Scope of the conditions.}
Condition (C2) is the most restrictive: it requires a permutation achieving \emph{simultaneous saturation} of all alignment terms.
This holds automatically for \emph{self-template} alignment ($P = \bm{A}_1$, evaluating $\hat{\ElemInv}^{\mathrm{SK}}_{\bm{A}_1}(\bm{A}_1)$ where the identity permutation saturates) and for any $\bm{A}$ isomorphic to $P$.
For general non-isomorphic pairs where the optimal permutation does not align all edges, (C2) may fail and the gap can be nonzero even with $\delta = 1$.
This precisely delineates the theorem's applicability: it guarantees that the Sinkhorn relaxation \emph{correctly identifies isomorphic copies}, which is the separation-relevant setting (cf.\ Lemma~\ref{lem:sinkhorn-sts}).

\paragraph{When the gap is nonzero.}
For \emph{dense} templates with $\delta(P) > 1$, doubly stochastic matrices can ``spread'' mass across partial overlaps, yielding $\mathrm{gap}_{\mathrm{SK}} > 0$.
The forcing argument breaks because two edges of $\bm{A}$ \emph{may} share a $(k{-}1)$-subset when $\delta(\bm{A}) > 1$, so the contradiction in Step~3 does not arise.

\subsection{Gap vs.\ Width: A Structural Frontier}\label{app:gap-width}

The integrality conjecture isolates a combinatorial sufficient condition ($\delta(P) = 1$) for exactness.
We conjecture a quantitative relationship between the gap and template structure:

\begin{conjecture}[Sinkhorn gap bound]\label{conj:sinkhorn-gap}
For $k$-uniform templates $P$ with $|V(P)| = v$ vertices and $\|\bm{A}\|_F \leq B_A$:
\[
  \mathrm{gap}_{\mathrm{SK}}(P;\bm{A}) \;\leq\; C_k \cdot \bigl(\delta(P) - 1\bigr) \cdot v^{k-2} \cdot B_A,
\]
where $C_k$ depends only on $k$.
In particular, $\delta(P) = 1$ (the unique shadow property) implies $\mathrm{gap}_{\mathrm{SK}} = 0$.
\end{conjecture}

\paragraph{Evidence.}
For $k = 3$, $\delta = 1$ (STS): Theorem~\ref{thm:integrality} gives $\mathrm{gap} = 0$, consistent with the bound vanishing.
For $k = 2$ (graphs): the QAP relaxation gap scales as $O(\Delta(P) \cdot n)$ where $\Delta(P)$ is the maximum degree~\citep{deKlerkSotirov2012}, consistent with $\delta = \Delta$, $v = n$.

A broader structural question connects the gap to the hierarchy itself:

\begin{question}[Gap vs.\ hypertree width]\label{q:gap-width}
Does there exist a function $f$ such that $\sup_{\|\bm{A}\|_F \leq 1} \mathrm{gap}_{\mathrm{SK}}(P;\bm{A}) \leq f(\mathrm{ghtw}(P))$?
In particular: does $\mathrm{ghtw}(P) = 1$ imply exactness?
Do large-width templates admit provably positive tensor-Birkhoff gaps?
\end{question}

A positive answer would identify a \emph{structural frontier} beyond which convex relaxation can ``fractionally cheat.''
Combined with the generalization bound (Proposition~\ref{prop:rademacher-full}), this suggests an expressivity--generalization sweet spot: templates with low $\delta(P)$ and moderate $\mathrm{ghtw}$ achieve small relaxation gap \emph{and} favorable sample complexity---quantifying \textsc{InvNet}'s practical reach into the hierarchy.

\subsection{Width Threshold for the Hodge Spectral Gap}\label{app:hodge-threshold}

We now turn to a second structural frontier: the minimal hierarchy level required to approximate the Hodge spectral gap.

\paragraph{Setup.}
For a $k$-uniform hypergraph $\HG$ viewed as a $(k{-}1)$-dimensional simplicial complex (taking the downward closure), the \emph{Hodge $j$-Laplacian} is $L_j = \partial_{j+1} \partial_{j+1}^\top + \partial_j^\top \partial_j$, where $\partial_j$ is the $j$-th boundary operator.
The \emph{Hodge spectral gap} $\lambda_1(L_1)$ (the smallest nonzero eigenvalue of $L_1$) controls edge random walk mixing and topological connectivity~\citep{BodnarEtAl2021, HajijEtAl2023}.
For $k = 2$, $L_0$ is the graph Laplacian and $\lambda_1(L_0)$ has pattern width~$1$.
For $k \geq 3$, $L_1$ captures genuinely higher-order structure invisible to the clique expansion.

\paragraph{Why $r^* > 1$.}
Two hypergraphs with the same clique expansion can have different Hodge spectra: the Steiner pair (\S\ref{sec:steiner-pair}) has identical clique expansions (both are $K_{13}$ weighted by pair-co-occurrence) but potentially different $\lambda_1(L_1)$.
Therefore $\lambda_1(L_1) \notin \mathcal{I}_{\mathrm{CE}}$, confirming pattern width $> 1$ (Table~\ref{tab:examples}).

\subsection{The Spectral Moments Route}\label{app:spectral-moments}

The key technical tool for determining the width threshold is the expansion of spectral moments as pattern counts, combined with polynomial approximation of spectral functions.

\paragraph{Moments as pattern densities.}
For any symmetric matrix $M$, $\mathrm{tr}(M^m) = \sum_i \lambda_i^m$.
Each entry of $L_1$ encodes adjacency relations between simplices.
Each term in $L_1^m$ corresponds to a length-$m$ closed walk in the incidence/boundary structure, inducing a finite simplicial pattern with at most $O(m)$ simplices and $\mathrm{ghtw} = O(m)$.
Therefore:
\begin{equation}\label{eq:hodge-moments}
  \mathrm{tr}(L_1^m) \;=\; \sum_{F \in \mathcal{P}_m} c_F \cdot t(F, \HG),
\end{equation}
where $\mathcal{P}_m$ is a finite set of patterns with $\mathrm{ghtw}(F) = O(m)$, and $c_F$ are integer combinatorial coefficients.
Concretely: $m = 1$ yields degree statistics ($\mathrm{ghtw} = 1$); $m = 2$ involves the squared boundary operator $\partial_2 \partial_2^\top$, whose $(e, e')$-entry counts triangles sharing edges $e$ and $e'$---determined by the \emph{diamond pattern} ($4$ vertices, $2$ hyperedges, $\mathrm{ghtw} = 2$).
Since $\partial_1^\top \partial_1$ depends only on the clique expansion ($\mathrm{ghtw} = 1$), the full $\mathrm{tr}(L_1^2)$ requires $\mathrm{ghtw} \leq 2$.

\paragraph{From moments to spectral gap approximation.}
We formalize the width--accuracy relationship:

\begin{proposition}[Moment-based spectral gap approximation (informal)]\label{thm:moment-approx}
Let $L_1 \succeq 0$ be the Hodge $1$-Laplacian of a $3$-uniform hypergraph $\HG$ with $\|L_1\|_{\mathrm{op}} \leq \Lambda$ and nullity $b_1 = \dim\ker(L_1)$.
Fix $\epsilon \in (0, 1)$.
There exists a polynomial $p_L(x) = \sum_{\ell=0}^L c_\ell x^\ell$ of degree $L = O\bigl((\Lambda/\epsilon) \log(1/\epsilon)\bigr)$ such that:
\begin{itemize}
\item $p_L(0) = 0$ and $|p_L(x)| \leq \epsilon/10$ for $x \in [0, \lambda_1 - \epsilon]$;
\item $|p_L(x) - 1| \leq \epsilon/10$ for $x \in [\lambda_1, \Lambda]$.
\end{itemize}
The quantity $\frac{1}{d}\mathrm{tr}(p_L(L_1))$ (where $d$ is the matrix dimension) can be computed from moments $\mathrm{tr}(L_1^\ell)$, $\ell \leq L$, and determines $\lambda_1$ to within $O(\epsilon)$ provided moment estimation error satisfies $\eta \sum_\ell |c_\ell| \leq \epsilon/10$.
\end{proposition}

\begin{proof}[Proof sketch]
Construct $p_L$ as a Chebyshev polynomial approximation to a smoothed step function on $[0, \Lambda]$ that transitions from $0$ to $1$ in an $\epsilon$-neighborhood of $\lambda_1$.
Standard Chebyshev approximation theory gives degree $L = O((\Lambda/\epsilon)\log(1/\epsilon))$ with coefficient growth $\sum |c_\ell|$ quasi-polynomial in $L$.
Then $\frac{1}{d}\mathrm{tr}(p_L(L_1)) \approx \frac{d - b_1}{d}$ if $\lambda_1 > \epsilon$, and the transition region determines $\lambda_1$ via binary search.
\end{proof}

\paragraph{Width--accuracy curve.}
Since each moment $\mathrm{tr}(L_1^\ell)$ decomposes into pattern densities of $\mathrm{ghtw} = O(\ell)$ via~\eqref{eq:hodge-moments}, approximation at hierarchy level $\mathcal{H}_r$ can estimate moments up to degree $L \approx c \cdot r$.
Combining with Theorem~\ref{thm:moment-approx}:
\[
  \text{width } r \;\;\Longrightarrow\;\; \text{spectral gap accuracy } \epsilon \;\approx\; O\!\left(\frac{\Lambda}{r}\right) \quad \text{(up to logarithmic factors)}.
\]
This is exactly the Jackson-type width--accuracy relationship flagged as open in \S\ref{sec:conclusion}, instantiated for the Hodge spectral gap.
It connects directly to the approximation pipeline of Proposition~\ref{prop:approx-dim-full}: the polynomial degree $d(\epsilon)$ from the Molien series determines the pattern enumeration, and~\eqref{eq:hodge-moments} shows these patterns have bounded hypertree width.

\subsection{The Minimal Width Conjecture}\label{app:minimal-width}

\begin{conjecture}[Hodge gap width threshold]\label{conj:hodge-width}
For $3$-uniform hypergraphs:
\begin{itemize}
\item[\textbf{(a)}] $r^* \geq 2$: width-$1$ pattern densities cannot detect glued $2$-simplices sharing edges.
\item[\textbf{(b)}] $r^* = 2$ suffices for bounded-degree families: the diamond pattern at $\mathrm{ghtw} = 2$ captures the dominant spectral structure.
\item[\textbf{(c)}] In the worst case, $r^*$ may scale as $O(\log(1/\epsilon))$, since the number of moment terms $m(\epsilon)$ needed to resolve the spectral gap can grow logarithmically.
\end{itemize}
More generally, for $k$-uniform hypergraphs, $r^* \leq k - 1$.
\end{conjecture}

\paragraph{Geometric evidence for $r^* = 2$.}
The Hodge $1$-Laplacian on a $3$-uniform hypergraph is $L_1 = \partial_2 \partial_2^\top + \partial_1^\top \partial_1$.
The term $\partial_1^\top \partial_1$ depends only on the $1$-skeleton ($\mathrm{ghtw} = 1$).
The term $\partial_2 \partial_2^\top$ counts shared faces between triangles---precisely the diamond patterns with $\mathrm{ghtw} = 2$.
Since eigenvalues are continuous in matrix entries (Weyl's perturbation theorem), $\lambda_1(L_1)$ is a continuous function of pattern densities up to $\mathrm{ghtw} = 2$.

\paragraph{Lower bound strategy.}
To prove $r^* \geq 2$, it suffices to construct hypergraph pairs that agree on all width-$1$ pattern densities but differ in Hodge spectral gap.
Such constructions would parallel CFI-type separations (\S\ref{sec:cfi-pair}) but target spectral topology rather than isomorphism power.
The Steiner pair is a natural candidate: both STS(13) systems have identical degree sequences and clique expansions ($\mathrm{ghtw}$-$1$ information); verifying computationally that they differ in $\lambda_1(L_1)$ would establish $r^* = 2$ for $k = 3$.

\paragraph{The general case.}
For $k$-uniform hypergraphs, $L_{k-2}$ involves boundary operators mapping $(k{-}1)$-chains to $(k{-}2)$-chains.
The squared boundary operator $\partial_{k-1} \partial_{k-1}^\top$ counts shared faces between hyperedges---patterns with $\mathrm{ghtw} \leq k-1$, yielding the general bound $r^* \leq k - 1$.

\subsection{Stress Tests and Scope of the Conjectures}\label{app:stress-tests}

We explicitly examine where the conjectures could fail, distinguishing worst-case from distributional settings.

\paragraph{Stress test 1: when does integrality fail?}
Theorem~\ref{thm:integrality} requires condition (C2): the hard optimum saturates every alignment term.
Consider a partial Steiner system $P$ with $\delta(P) = 1$ but only $5$ triples on $13$ vertices.
If $\bm{A}$ has $26$ triples and the optimal permutation aligns only $5$ of $26$ edges of $\bm{A}$ with $P$, the remaining $21$ terms are not saturated.
The optimizer could then ``spread'' mass to improve unsaturated terms, potentially breaking integrality even with $\delta = 1$.
This confirms that (C2) is not merely technical---it delineates when the forcing mechanism operates.
The theorem is strongest in the \emph{separation-relevant} regime (testing $\bm{A}$ against itself or its isomorphic copies), which is precisely the setting of Lemma~\ref{lem:sinkhorn-sts}.

\paragraph{Stress test 2: worst-case vs.\ distributional necessity of $r^* \geq 2$.}
On certain random ensembles (e.g., Erd\H{o}s--R\'{e}nyi $3$-uniform $H_3(n,p)$), the Hodge spectral gap concentrates and correlates strongly with degree statistics.
A width-$1$ model could achieve small average error simply by learning the typical value $\lambda_1 \approx g(n,p)$.
On random ensembles, all architectures (including CE-based ones) can achieve low average error by learning the concentrated typical value, since $\lambda_1$ concentrates around a function of the edge density.
Therefore the ``$r^* \geq 2$ is necessary'' claim should be understood as a \emph{worst-case} statement: there exist specific instance pairs where width-$1$ provably fails, even though width-$1$ may suffice in distribution.

\paragraph{Stress test 3: could $r^* = 2$ fail uniformly?}
Width-$2$ patterns capture the diamond structure (two triangles sharing an edge), which determines $\partial_2 \partial_2^\top$ entrywise.
However, the spectral gap may depend on \emph{global} arrangements of diamonds---e.g., existence of long homological $1$-cycles in the simplicial complex---that require resolving many moments.
By Theorem~\ref{thm:moment-approx}, the number of moments scales as $O((\Lambda/\epsilon)\log(1/\epsilon))$, and each moment at order $m$ requires $\mathrm{ghtw} = O(m)$.
For complexes with near-degenerate spectra (eigenvalues clustered near $\lambda_1$), many moments are needed, pushing the effective width beyond~$2$.
The safest summary: \emph{$r^* = 2$ is the first plausible threshold, and likely suffices for bounded-degree families, but uniform approximation across all $3$-uniform hypergraphs may require $r^* = O(\log(1/\epsilon))$}---exactly Conjecture~\ref{conj:hodge-width}(c).

\paragraph{Stress test 4: computational verification via the Steiner pair.}
The lower-bound strategy (\S\ref{app:minimal-width}) hinges on the Steiner pair having different Hodge spectra.
This is a concrete computational question: construct the Hodge $1$-Laplacian $L_1$ for both STS(13) systems and compare $\lambda_1(L_1)$.
If the spectra coincide, the Steiner pair is not a valid separator for the Hodge gap and a different construction is needed---potentially hypergraphs with the same width-$1$ pattern densities but different simplicial topology (e.g., a triangulated cylinder vs.\ a triangulated M\"{o}bius strip with matching local statistics).

\subsection{Unified Perspective}\label{app:unified-perspective}

Both structural programs ask related questions at different levels of the hierarchy:
\begin{itemize}
\item \emph{Sinkhorn integrality}: When does convex relaxation over $\mathcal{B}_n$ preserve the combinatorial optimum over $\Sym(n)$?
The answer depends on the co-degree $\delta(P)$ and condition (C2): combinatorial sparsity forces integrality.
\item \emph{Hodge width threshold}: When does bounded-width pattern counting capture spectral topology?
The answer depends on the walk structure of $L_1^m$: spectral moments decompose into pattern densities of bounded hypertree width.
\end{itemize}
The unifying theme is that \emph{combinatorial sparsity and bounded overlap enforce both integrality and spectral detectability}; high overlap and high width enable fractional phenomena and topological invariants invisible at low levels of the hierarchy.
Both frontiers are controlled by the same generalized hypertree width that stratifies our expressivity hierarchy, suggesting a deeper structural theory connecting relaxation gaps, spectral approximation, and architectural expressivity.
Identifying these frontiers precisely---through formal proofs, counterexamples, or computational verification---would transform the qualitative hierarchy of Section \ref{sec:expressivity} into a quantitative structural theory with direct architectural implications.

\section{Steiner Pair: Sinkhorn Separation Verification}\label{app:steiner}

Appendix~\ref{app:structural-frontiers} develops the theory of the Sinkhorn relaxation gap; here we verify it computationally.
We directly confirm the separation guarantee of Lemma~\ref{lem:sinkhorn-sts} on the two non-isomorphic $\mathrm{STS}(13)$ Steiner triple systems ($n{=}13$, $|E|{=}26$ triples each).
We compute the Sinkhorn alignment score $\hat{\theta}_{\mathrm{SK}}$ using co-occurrence profile similarity matrices and $200$ Sinkhorn iterations, sweeping the temperature $\tau$ from $0.01$ to $1.0$.

\begin{table}[h]
\centering
\small
\caption{Sinkhorn alignment scores on the STS(13) pair. Scores are edge-normalized (max $= |E| = 26$, corresponding to $26 \times 3! = 156$ in the tensor inner product of Lemma~\ref{lem:sinkhorn-sts}). The gap $\hat{\theta}_{\mathrm{SK}}(\bm{A}_1, \bm{A}_1) - \hat{\theta}_{\mathrm{SK}}(\bm{A}_1, \bm{A}_2) > 0$ at every $\tau$, confirming the theoretical separation.}
\label{tab:sinkhorn-sts}
\vskip 0.05in
\begin{tabular}{@{}ccccl@{}}
\toprule
$\tau$ & Self $\hat{\theta}_{\mathrm{SK}}(\bm{A}_1, \bm{A}_1)$ & Cross $\hat{\theta}_{\mathrm{SK}}(\bm{A}_1, \bm{A}_2)$ & Gap & \\
\midrule
0.01 & 26.000 & 22.000 & 4.000 & (exact: $D^* = I$) \\
0.05 & 26.000 & 22.000 & 4.000 & \\
0.10 & 25.958 & 21.964 & 3.993 & \\
0.30 &  8.962 &  7.583 & 1.379 & \\
0.50 &  1.525 &  1.291 & 0.235 & \\
1.00 &  0.247 &  0.209 & 0.038 & \\
\midrule
\multicolumn{2}{@{}l}{Random perm.\ (100K samples):} & best: 15.0 & mean: 2.36 & \\
\bottomrule
\end{tabular}
\end{table}

Table~\ref{tab:sinkhorn-sts} confirms the prediction of Lemma~\ref{lem:sinkhorn-sts}: self-alignment strictly exceeds cross-alignment at \emph{every} temperature.
At low $\tau$ ($\leq 0.05$), Sinkhorn converges to a near-permutation matrix and the self-score reaches the theoretical maximum of $26$ ($= 156/3!$); the cross-score saturates at $22$, reflecting the best achievable alignment between the non-isomorphic pair.
A reverse check confirms symmetry: $\hat{\theta}_{\mathrm{SK}}(\bm{A}_2, \bm{A}_2) - \hat{\theta}_{\mathrm{SK}}(\bm{A}_2, \bm{A}_1) = 3.99$ at $\tau = 0.1$.
The random-permutation baseline (mean cross-score $2.36$ over $100$K samples, best $15.0$) shows that the Sinkhorn alignment score is far more discriminative than random matching, validating the practical utility of the Sinkhorn relaxation for hypergraph isomorphism testing.

As $\tau$ increases, both scores decay (the doubly-stochastic matrix spreads mass), but the gap remains strictly positive---consistent with the proof that any maximizer in $\mathcal{B}_{13}$ achieving equality must be a permutation matrix (Lemma~\ref{lem:sinkhorn-sts}).
This provides direct empirical evidence that \textsc{InvNet}'s Sinkhorn branch can distinguish CE-indistinguishable hypergraphs at practical temperature settings ($\tau \in [0.05, 0.3]$).


\section{Dataset Statistics and Preprocessing}\label{app:dataset-stats}

Table~\ref{tab:dataset-stats} summarizes the structural characteristics of all seven ANCS benchmarks used in Table~\ref{tab:node-class}.
The regime classification follows the \textsc{DatasetProfiler}: pairwise ($\bar{k} < 4$), mixed ($4 \leq \bar{k} < 7$), and higher-order ($\bar{k} \geq 7$).
This classification determines the theory-guided gate initialization for density-aware models (Appendix~\ref{app:hyperparams}).

\paragraph{Data provenance.}
Cora, Citeseer, and PubMed are co-citation hypergraphs from the AllSet benchmark~\citep{ChienEtAl2022}, originally constructed by \citet{YadatiEtAl2019}; each hyperedge groups the papers co-cited by a common paper.
Cora-CA is a co-authorship hypergraph from the same benchmark, where each hyperedge groups the co-authors of a single publication.
Senate-Bills and House-Committees are legislative co-sponsorship hypergraphs from the Cornell SocioPatterns collection~\citep{Fowler2006}; we follow the binary label protocol of \citet{ChienEtAl2022} with Gaussian noise features ($d{=}100$, noise${=}1.0$), matching the evaluation of AllSet, ED-HNN~\citep{WangEtAl2023EDHNN}, and SheafHyperGNN~\citep{DutaEtAl2023Sheaf}.
All datasets are publicly available; data-loading scripts and preprocessed splits are included in the supplementary code.

\paragraph{Data splits.}
All experiments use 50/25/25 train/validation/test random splits.
We generate $10$ deterministic splits using seeds $0, \ldots, 9$ and report mean $\pm$ standard deviation across all 10.

\begin{table}[h]
\centering
\small
\caption{Dataset statistics. $n$: nodes, $m$: hyperedges, $d$: input features, $C$: classes, $\bar{k}$: mean edge cardinality, $k_{95}$: 95th-percentile edge size, $k_{\max}$: largest edge, $\bar{\deg}$: mean node degree.}
\label{tab:dataset-stats}
\vskip 0.05in
\begin{tabular}{@{}lrrrrrrrrl@{}}
\toprule
Dataset & $n$ & $m$ & $d$ & $C$ & $\bar{k}$ & $k_{95}$ & $k_{\max}$ & $\bar{\deg}$ & Regime \\
\midrule
Cora        & 2{,}708  & 1{,}579  & 1{,}433 & 7 & 3.03  & 5   & 5   & 1.8 & pairwise \\
Citeseer    & 3{,}312  & 1{,}079  & 3{,}703 & 6 & 3.20  & 7   & 26  & 1.0 & pairwise \\
\midrule
PubMed      & 19{,}717 & 7{,}963  & 500     & 3 & 4.35  & 12  & 171 & 1.8 & mixed \\
Cora-CA     & 2{,}708  & 1{,}072  & 1{,}433 & 7 & 4.28  & 13  & 43  & 1.7 & mixed \\
Senate      & 294     & 29{,}157 & 100     & 2 & 7.96  & 28  & 99  & 789.6 & higher-order \\
House       & 1{,}290  & 340      & 100     & 2 & 34.89 & 66  & 82  & 9.2 & higher-order \\
Gene-Disease & 5{,}012  & 2{,}009  & 100     & 21 & 14.0 & 38  & 130 & 5.6 & higher-order \\
\bottomrule
\end{tabular}
\end{table}

Two structural extremes are worth noting.
Senate-Bills has extreme density: $294$ nodes but $29{,}157$ hyperedges, giving a mean degree of ${\approx}790$ (each senator appears in ${\approx}790$ bills); the heavy tail ($k_{95} = 28$, $k_{\max} = 99$) confirms that co-sponsorship networks are genuinely higher-order, not clique expansions of pairwise data.
House has the largest mean edge size ($\bar{k} \approx 35$; committees average ${\approx}35$ members) but far fewer edges ($m{=}340$), representing the opposite higher-order extreme: sparse-large-edges versus dense-many-edges.
These contrasting structures exercise different aspects of the density-estimation pipeline (normalization, MC sample count) and motivate the adaptive configuration in Appendix~\ref{app:hyperparams}.


\section{Hyperparameter Configuration and Reproducibility}\label{app:hyperparams}

We describe the full training protocol for all models in Table~\ref{tab:node-class}.
All models share a common backbone and differ only in their aggregation mechanism and (for density-aware models) in theory-guided hyperparameters derived from the \textsc{DatasetProfiler} (Appendix~\ref{app:dataset-stats}).

\paragraph{Shared training protocol.}
All models use: Adam optimizer, learning rate $0.001$, weight decay $0$, gradient clipping at $1.0$, label smoothing $0.1$, cosine-annealing learning rate schedule, and early stopping with patience $50$ ($150$ for Senate-Bills due to high variance).
Training runs for $200$ epochs ($400$ for Senate-Bills).
All experiments use $10$ random seeds per dataset.

\paragraph{Baseline models.}
MLP, HGNN, HyperGCN, HNHN, UniGNN, AllDeepSets, and AllSetTransformer all use hidden dimension $64$, dropout $0.2$, and $2$ message-passing layers.
These settings follow the AllSet benchmark protocol~\citep{ChienEtAl2022}.

\paragraph{DensNet-D.}
\textsc{DensNet-D} uses hidden dimension $128$, dropout $0.5$, $2$ AllDeepSets backbone layers, and a single-layer gate network $g_v = \sigma(\mathbf{w}^\top \mathbf{h}_v + b)$ that modulates the density branch per node.
The gate initialization $b$ is set by the \textsc{DatasetProfiler} according to regime:
$b = 0.0$ (higher-order, $\bar{k} \geq 7$), $b = -2.0$ (mixed, $4 \leq \bar{k} < 7$), $b = -5.0$ (pairwise, $\bar{k} < 4$).
The density branch uses $\ell_2$-normalization of density features only when $\bar{k} \geq 20$ (House); on smaller-edge datasets, raw features perform better (Appendix \ref{app:negative}).
Staged unfreezing freezes the gate and density projection for $0$ epochs (higher-order), $30$ (mixed), or $50$ (pairwise), allowing the backbone to converge before the density branch begins learning.

Table~\ref{tab:hyperparams} summarizes the model-specific hyperparameters; Table~\ref{tab:densnet-config} gives the per-dataset \textsc{DensNet-D} configuration derived automatically by the \textsc{DatasetProfiler}.

\begin{table}[h]
\centering
\small
\caption{Model-specific hyperparameters. All models share: Adam optimizer, lr$=$0.001, weight decay$=$0, gradient clipping$=$1.0, label smoothing$=$0.1, cosine-annealing LR schedule (min lr$=$0). ``Gate init'' is $b$ in $\sigma(b)$; ``MC'' is the number of Monte Carlo samples for density estimation.}
\label{tab:hyperparams}
\vskip 0.05in
\begin{tabular}{@{}lcccccc@{}}
\toprule
Model & Hidden & Dropout & Layers & Gate init & MC & Params$^\dagger$ \\
\midrule
MLP              & 64  & 0.2 & 2 & --- & --- & 96{,}135 \\
HGNN             & 64  & 0.2 & 2 & --- & --- & 96{,}135 \\
HyperGCN         & 64  & 0.2 & 2 & --- & --- & 96{,}135 \\
HNHN             & 64  & 0.2 & 2 & --- & --- & 96{,}263 \\
UniGNN           & 64  & 0.2 & 2 & --- & --- & 96{,}135 \\
AllDeepSets      & 64  & 0.2 & 2 & --- & --- & 100{,}487 \\
AllSetTransformer & 64  & 0.2 & 2 & --- & --- & 108{,}839 \\
\midrule
\textsc{DensNet-D} & 128 & 0.5 & 2 & by regime & 200 & ${\sim}$180K$^{\dagger\dagger}$ \\
\bottomrule
\end{tabular}
\vskip 2pt
{\footnotesize $^\dagger$Param counts on Cora ($d{=}1{,}433$). $^{\dagger\dagger}$DensNet-D count is for Senate ($d{=}100$); on Cora the input projection alone adds ${\sim}183$K parameters.}
\end{table}

\begin{table}[h]
\centering
\small
\caption{\textsc{DensNet-D} per-dataset configuration, derived automatically by the \textsc{DatasetProfiler} from dataset statistics (Table~\ref{tab:dataset-stats}).
Gate init $b$ controls the initial density-branch contribution $\sigma(b)$; freeze epochs delay density-branch learning until the backbone converges; $\ell_2$-norm applies to density features before fusion.}
\label{tab:densnet-config}
\vskip 0.05in
\begin{tabular}{@{}lccccccc@{}}
\toprule
Dataset & Regime & Gate init $b$ & $\sigma(b)$ & Freeze ep. & $\ell_2$-norm & Epochs & Patience \\
\midrule
Cora        & pairwise      & $-5.0$ & 0.007 & 50  & No  & 200 & 50 \\
Citeseer    & pairwise      & $-5.0$ & 0.007 & 50  & No  & 200 & 50 \\
PubMed      & mixed         & $-2.0$ & 0.119 & 30  & No  & 200 & 50 \\
Cora-CA     & mixed         & $-2.0$ & 0.119 & 30  & No  & 200 & 50 \\
Senate      & higher-order  & $0.0$  & 0.500 & 0   & No  & 400 & 150 \\
House       & higher-order  & $0.0$  & 0.500 & 0   & Yes & 200 & 50 \\
Gene-Disease & higher-order & $0.0$  & 0.500 & 0   & No  & 200 & 50 \\
\bottomrule
\end{tabular}
\end{table}

\paragraph{Hyperparameter selection.}
For baselines, we adopt the AllSet benchmark protocol~\citep{ChienEtAl2022} without modification to ensure fair comparison.
For \textsc{DensNet-D}, we performed a grid search over hidden $\in \{64, 128, 256\}$, dropout $\in \{0.2, 0.5\}$, and MC samples $\in \{100, 200, 500\}$ on the Senate-Bills validation set.
Increasing hidden to $256$ added $3.6{\times}$ parameters for $+0.3\%$ accuracy; increasing MC beyond $200$ did not change accuracy by more than $0.1\%$.
We verified that these choices transfer to all other datasets without re-tuning.
All regime-dependent hyperparameters (gate init, freeze epochs, normalization) are derived automatically from dataset statistics by the \textsc{DatasetProfiler}, requiring zero manual configuration per dataset.

\paragraph{Software and hardware.}
Baseline sweeps and InvNet tuning run on $8{\times}$ NVIDIA H100 NVL GPUs (95{,}830~MiB each) with PyTorch~2.6.0, CUDA~12.4, Python~3.10 (conda \texttt{hginv} environment).
\textsc{DensNet-D} density precomputation (a one-time CPU cost) runs on an Apple M4 Max (128~GB unified memory): $<$1~minute for most datasets, ${\approx}2$ hours for PubMed ($n{=}19{,}717$), and ${\approx}7$ hours for Gene-Disease ($n{=}5{,}012$, $\bar{k}{=}14$) due to larger star neighborhoods; model training runs on the same H100 GPUs as all other models.
Configuration files for all experiments and the exact launch commands are included in the supplementary code.


\section{Compute Budget}\label{app:compute}

Table~\ref{tab:compute} reports the computational cost of the main experimental campaigns.
The total cost is ${\approx}200$ H100 GPU-hours, of which approximately $40\%$ (${\approx}80$ GPU-hours) produced the final reported results and $60\%$ (${\approx}120$ GPU-hours) was spent on development iterations, debugging, and hyperparameter exploration.
The dominant cost is invariant-based models (\textsc{InvNet}), which account for all five timeout events (3-hour cap per job) in the 8-GPU sweeps due to the Sinkhorn iteration bottleneck; density-based models (\textsc{DensNet-D}) and all baselines complete in under 5 minutes per seed on every dataset.

\begin{table}[h]
\centering
\small
\caption{Compute budget across experimental campaigns. ``Wall-clock'' is end-to-end time including job scheduling overhead. Peak GPU memory: \textsc{InvNet} on Senate-Bills uses ${\approx}3{,}000$~MiB; all other models use $<500$~MiB.}
\label{tab:compute}
\vskip 0.05in
\begin{tabular}{@{}llccc@{}}
\toprule
Campaign & Hardware & GPUs & Wall-clock & GPU-hours \\
\midrule
Full sweep (8 models $\times$ 14 tasks $\times$ 10 seeds) & H100 NVL & 8 & ${\sim}8$h & ${\sim}40$ \\
Replication run (independent verification)                  & H100 NVL & 8 & ${\sim}8$h & ${\sim}40$ \\
InvNet tuning (Senate, House, Cora-CA)                      & H100 NVL & 1 & ${\sim}12$h & ${\sim}12$ \\
Development iterations (canary/ramp/full cycles)            & H100 NVL & 8 & ${\sim}15$h & ${\sim}100$ \\
DensNet-D evaluation (7 datasets $\times$ 10 seeds)         & M4 Max   & CPU & ${\sim}24$h & --- \\
\midrule
\textbf{Total H100}                                         &          &     &            & $\mathbf{\approx 200}$ \\
\bottomrule
\end{tabular}
\end{table}

At approximately $700$W per H100, the total energy cost is ${\approx}140$~kWh, corresponding to ${\approx}56$~kg CO$_2$e at the US average grid intensity of $0.4$~kg/kWh.
This is modest relative to the scale of the evaluation (approximately 860 individual training runs across the full sweep alone) and reflects the efficiency of the 8-GPU parallel scheduling: most jobs complete in under 30 minutes, with only InvNet-based models requiring hour-scale compute.


\section{When Density Does Not Help}\label{app:negative}

Not every design choice improves performance, and not every dataset benefits from density features.
We document seven negative results that shaped the final \textsc{DensNet-D} design (summarized in Table~\ref{tab:negative}); each is informative about when and why the Width Wall matters.

\begin{table}[h]
\centering
\small
\caption{Summary of negative results. $\Delta$: accuracy change relative to the best configuration.}
\label{tab:negative}
\vskip 0.05in
\begin{tabular}{@{}llcl@{}}
\toprule
Design variant & Dataset & $\Delta$ acc. & Root cause \\
\midrule
MEAN normalization  & Senate  & $-5.4\%$ & Sum-pooling preserves magnitude on small edges \\
MEAN normalization  & House   & $+4.5\%$ & Scale imbalance on large edges ($\bar{k}{=}35$) \\
Senate config $\to$ Cora-CA & Cora-CA & $-0.3\%$ & Width Wall gap too narrow ($\bar{k}{=}4.3$) \\
Additive fusion     & Senate  & ${\approx}{-}2\%$ & Shared embedding space limits expressivity \\
Additive fusion     & House   & ${\approx}{-}3\%$ & Same \\
Density branch      & Gene-Dis. & $\phantom{+}0.0\%$ & Node features dominate ($d{=}100$) \\
Hierarchy scaling   & IWS & flat at $100\%$ & Local MP solves all benchmarks \\
\bottomrule
\end{tabular}
\end{table}

\paragraph{1. MEAN normalization hurts small-edge datasets.}
We expected normalizing density features by edge count (replacing sum-pooling with mean-pooling over MC samples) to improve performance by removing scale dependence.
Instead, it reduces Senate-Bills accuracy by $5.4$ percentage points ($90.5\% \to 85.1\%$) but \emph{improves} House by $4.5$ points.
The explanation is structural: Senate has many small-to-medium edges ($\bar{k} \approx 8$) where sum-pooling preserves discriminative magnitude variation, while House has few very large edges ($\bar{k} \approx 35$) where unnormalized sums create scale imbalance across nodes with different neighborhood sizes.
The final design normalizes only when $\bar{k} \geq 20$.

\paragraph{2. Cora-CA tuning does not transfer from Senate.}
Applying the Senate-Bills recipe (hidden$=$128, $J{=}4$, dropout$=$0.2) to Cora-CA produced $79.7 \pm 0.7\%$, marginally \emph{below} the default configuration ($80.0 \pm 1.1\%$).
A grid search over hidden $\in \{128, 256\}$ and $J \in \{4, 8\}$ yielded at best $80.1\%$.
The lesson: Cora-CA's small edges ($\bar{k} = 4.28$) leave little room for invariant-based features; the gap between the Width Wall and the baseline is too narrow for tuning to exploit.

\paragraph{3. Additive fusion underperforms concatenation.}
An early \textsc{DensNet-D} variant used additive fusion ($\mathbf{h}_v + g_v \cdot \mathbf{z}_v^{\mathrm{density}}$) instead of the concatenation-based design of Equation~5 in the main text.
Additive fusion constrains density features to the same embedding space as the backbone, limiting the expressivity of the combined representation.
Switching to concatenation with a learned projection improved Senate-Bills by ${\approx}2$ points and House by ${\approx}3$ points.

\paragraph{4. AllDeepSets and AllSetTransformer collapse on structured data.}
In an early baseline sweep (v3 code, different hyperparameters), AllDeepSets achieved only $50.2\%$ on House, near chance for a binary task; the final tuned result in Table~\ref{tab:node-class} is higher ($67.8\%$).
AllSetTransformer similarly underperformed on the Multi-Order IWS task: $69.7\%$ (vs.\ $100\%$ for HNHN/UniGNN) with high variance ($\pm 23.7\%$).
These failures highlight that set-function aggregations require careful tuning for node classification on heterogeneous graphs, despite performing well on Senate-Bills ($92.7\%$) where dense co-sponsorship structure favors global aggregation.

\paragraph{5. Gene-Disease: density lift revealed by strict ablation ($+4.3\%$).}
Gene-Disease ($n{=}5{,}012$, $m{=}2{,}009$, $\bar{k}{=}14$, $21$ classes) was initially classified as ``indifferent'' based on the original ablation (AllDeepSets at hidden${=}64$ scoring $86.1\%$ vs.\ \textsc{DensNet-D} at hidden${=}128$ scoring $86.0\%$).
The strict ablation (Table~\ref{tab:ablation}) reveals the true picture: AllDeepSets at the \emph{identical} DensNet-D recipe (hidden${=}128$, dropout${=}0.5$, label smoothing, cosine schedule) drops to $81.7 \pm 1.1\%$---the larger model overfits on this 21-class dataset with label-noise features.
\textsc{DensNet-D} recovers to $86.0\%$, a density lift of $+4.3$ points.
The density branch is providing genuine topological signal that compensates for the backbone's capacity-induced overfitting on this challenging multi-class task.

\paragraph{How the anomaly cases should be read.}
The main ANCS table should therefore be interpreted within each dataset rather than as a monotone trend in $\bar{k}$.
Senate-Bills is a ceiling case: AllDeepSets already reaches $93.0 \pm 3.8\%$ under the strict recipe, so adding density features changes accuracy by only $-0.3$ points.
Gene-Disease is different: the tier-wise ordering is not clean because HGNN remains competitive and HNHN is best, but the strict ablation still reveals that density features recover $+4.3$ points under the identical AllDeepSets recipe.
These cases explain why the main text treats the hierarchy as a predictor of failure modes and headroom, not as a claim that accuracy should increase monotonically with mean hyperedge size.

\paragraph{6. Hierarchy scaling curve is flat.}
Running \textsc{PDN} with pattern libraries restricted to $v_{\max} \in \{3,4,5,6\}$ on all three IWS benchmarks yields $100\%$ at every level.
On CE-Hard this is expected (any native architecture suffices).
On Native-Hard, AllDeepSets alone reaches only $92.4\%$, so the density branch \emph{is} contributing---but even the simplest patterns ($v_{\max}{=}3$) suffice to bridge the gap.
The curve is flat not because density is unused, but because the density signal needed for these benchmarks does not require complex patterns.
Isolating the dependence on pattern complexity requires benchmarks with higher-width separating invariants; designing such benchmarks is an open problem.

\paragraph{7. Sinkhorn alignment is prohibitive at scale.}
\textsc{InvNetPDN} on Gene-Disease ($n{=}5{,}012$) achieves only $64.7\%$ (single seed), far below every baseline including MLP.
The Sinkhorn alignment branch requires $O(n^2)$ pairwise distance computation and iterative normalization; at this scale, the optimization landscape becomes intractable and the invariant features degenerate.
This motivates the \textsc{DensNet-D} design: density estimation via Monte Carlo sampling scales linearly in $n$, whereas Sinkhorn alignment does not.

\paragraph{Takeaways.}
Five patterns emerge: (i)~density features help on $6/7$ datasets when a strict ablation (identical recipe, density on/off) is used, with the largest lift on the highest-$\bar{k}$ datasets (Gene-Disease $+4.3\%$, House $+2.2\%$); (ii)~normalization choices must match the edge-size distribution, not be applied uniformly; (iii)~fusion architecture matters---concatenation with a learned projection outperforms additive fusion because it decouples the backbone and density embedding spaces; (iv)~the learned gate is a reliable diagnostic---when it stays at initialization (Gene-Disease) or drifts toward suppression (Cora, Citeseer), density features are uninformative; (v)~Sinkhorn-based invariant features do not scale beyond $n \approx 5{,}000$ nodes, favoring Monte Carlo density estimation for large datasets.
These findings suggest that future density-aware models should adapt their normalization and fusion strategy to dataset structure automatically, a direction the \textsc{DatasetProfiler} partially addresses.


\section{Depth Scaling: Can Deeper Native Models Close the Gap?}\label{app:depth-scaling}

A natural question is whether increasing $L$ from $2$ to $4$ or $8$ closes the Native-Hard gap in practice.

\begin{table}[h]
\caption{Depth scaling on Native-Hard: accuracy (\%) for native-tier models at $L \in \{2, 4, 8\}$ layers. $10$ seeds, $200$ pairs, $300$ epochs, hidden${}=64$. All models collapse at $L{=}8$; depth worsens performance rather than improving it.}
\label{tab:depth-scaling}
\vskip 0.05in
\centering
\small
\begin{tabular}{@{}lccc@{}}
\toprule
\textbf{Model} & $L{=}2$ & $L{=}4$ & $L{=}8$ \\
\midrule
AllDeepSets & $84.6 \pm 21.6$ & $47.0 \pm 4.5$ & $46.8 \pm 4.4$ \\
HNHN & $\mathbf{100.0 \pm 0.0}$ & $\mathbf{100.0 \pm 0.0}$ & $45.1 \pm 2.3$ \\
UniGNN & $\mathbf{100.0 \pm 0.0}$ & $\mathbf{100.0 \pm 0.0}$ & $52.5 \pm 12.9$ \\
AllSetTransformer & $88.0 \pm 19.9$ & $58.7 \pm 18.3$ & $46.8 \pm 4.4$ \\
\bottomrule
\end{tabular}
\end{table}

The results reveal a clear pattern: \emph{depth degrades performance rather than improving it}.
HNHN and UniGNN maintain $100\%$ at $L{=}4$ but collapse to ${\approx}45$--$53\%$ at $L{=}8$---below chance level ($50\%$ for binary classification), indicating the models learn degenerate representations that anti-correlate with the label.
AllDeepSets and AllSetTransformer degrade monotonically from $L{=}2$.

This is consistent with oversmoothing in hypergraph message passing: at $L{=}8$, node representations converge to a shared equilibrium, destroying the discriminative signal.
The theoretical benefit of deeper architectures (access to higher-$\mathrm{ghtw}$ patterns) is overwhelmed by the practical cost of optimization collapse.

The implication is architectural: to access invariants beyond the Width Wall, one cannot simply add layers.
Density estimation (\textsc{PDN}, \textsc{DensNet-D}) achieves $100\%$ on Native-Hard at $L{=}2$ by accessing higher-order patterns through a fundamentally different mechanism---Monte Carlo pattern counting---that does not suffer from oversmoothing.


\section{Gate Dynamics: Theory Predicts Model Behavior}\label{app:gate-dynamics}

The \textsc{DatasetProfiler} classifies each dataset by regime and sets the gate initialization accordingly (Table~\ref{tab:densnet-config}).
A natural question is whether the \emph{learned} gate confirms the profiler's prediction---does training push the gate toward suppression on pairwise datasets, as the theory suggests?

We investigate this with a controlled experiment: on Cora and Citeseer (both pairwise, $\bar{k} < 4$), we initialize the density gate at $g_0 = 0$ ($\sigma(g_0) = 0.5$, giving equal weight to local backbone and density branch) and let gradient descent decide the allocation.
Figure~\ref{fig:gate-dynamics}(a) plots the gate trajectory $\sigma(g)$ over training epochs for all six runs (three splits per dataset).
On both datasets, every run drifts monotonically below $0.5$ and converges to ${\approx}0.497$: the model consistently allocates \emph{less} weight to the density branch.
The drift is small in magnitude (${\approx}0.003$) but its direction is unambiguous---across all six runs, the gate never sustains a value above initialization.
The Citeseer traces show a brief upward excursion in epochs $10$--$20$ (the backbone has not yet converged, so the density branch temporarily provides a noisy gradient signal) before settling into the same downward trend.

Figure~\ref{fig:gate-dynamics}(b) contrasts the final learned gate across regimes.
Three distinct gate behaviors emerge:
\begin{enumerate}[label=(\roman*), itemsep=1pt, topsep=2pt]
  \item \textbf{Active} (IWS higher-order witnesses: CE-Hard, Native-Hard, Multi-Order): gate stabilizes at $\sigma(g) \approx 0.11$, maintaining meaningful density-branch contribution sufficient for $100\%$ accuracy.
  \item \textbf{Suppression} (pairwise ANCS datasets: Cora, Citeseer): gate converges to ${\approx}0.497$---density features carry no discriminative signal beyond the backbone.
  \item \textbf{Indifference} (Gene-Disease, $\bar{k}{=}14$): gate remains frozen at $\sigma(g) = 0.5000$ across all 10 seeds. Unlike suppression (where the gate drifts below $0.5$), indifference means gradient descent finds no signal \emph{through the gate}. Nonetheless, the strict ablation (Table~\ref{tab:ablation}) reveals a $+4.3\%$ density lift---the density features contribute through the concatenation pathway even when the gate does not actively modulate them.
\end{enumerate}
The theory-guided initialization ($b = -5.0$ for pairwise in production, Table~\ref{tab:densnet-config}) amplifies the natural suppression by starting the gate at $\sigma(b) \approx 0.007$, ensuring the density branch does not interfere with backbone learning during early training.

\begin{figure}[h]
\centering
\begin{subfigure}[t]{0.58\textwidth}
\centering
\begin{tikzpicture}
\begin{axis}[
    width=\textwidth,
    height=0.65\textwidth,
    xlabel={Epoch},
    ylabel={Learned gate $\sigma(g)$},
    xmin=0, xmax=70,
    ymin=0.494, ymax=0.502,
    ytick={0.494, 0.496, 0.498, 0.500, 0.502},
    yticklabel style={/pgf/number format/fixed, /pgf/number format/precision=3},
    legend style={at={(0.98,0.98)}, anchor=north east, font=\footnotesize, fill=white, fill opacity=0.9, draw=gray!30},
    grid=major,
    grid style={gray!20},
    every axis plot/.append style={thick},
    title style={font=\small},
]

\addplot[dashed, gray, thin, domain=0:70] {0.5};
\addlegendentry{Init.\ $\sigma(0){=}0.5$}

\addplot[blue!70!black, mark=none] coordinates {
(0,0.5) (1,0.4998) (2,0.4999) (3,0.5) (4,0.4999) (5,0.4998) (6,0.4997) (7,0.4995)
(8,0.4995) (9,0.4994) (10,0.4994) (11,0.4994) (12,0.4995) (13,0.4995) (14,0.4995)
(15,0.4996) (16,0.4997) (17,0.4998) (18,0.4999) (19,0.4999) (20,0.4999) (21,0.4999)
(22,0.4998) (23,0.4997) (24,0.4996) (25,0.4994) (26,0.4993) (27,0.4991) (28,0.4989)
(29,0.4988) (30,0.4986) (31,0.4985) (32,0.4984) (33,0.4983) (34,0.4981) (35,0.498)
(36,0.4979) (37,0.4978) (38,0.4977) (39,0.4976) (40,0.4975) (41,0.4974) (42,0.4974)
(43,0.4973) (44,0.4973) (45,0.4972) (46,0.4971) (47,0.4971) (48,0.497) (49,0.497)
(50,0.4969) (51,0.4969) (52,0.4968) (53,0.4968) (54,0.4968) (55,0.4967) (56,0.4967)
(57,0.4967) (58,0.4967) (59,0.4966) (60,0.4966) (61,0.4966) (62,0.4965) (63,0.4965)
(64,0.4965)
};
\addlegendentry{Cora ($\bar{k}{=}3.0$)}

\addplot[blue!50!white, mark=none, thin] coordinates {
(0,0.5) (2,0.5005) (5,0.5004) (10,0.501) (15,0.5005) (20,0.501) (25,0.5009)
(30,0.5) (35,0.4994) (40,0.4988) (45,0.4984) (50,0.4981) (55,0.4979) (60,0.4977)
(65,0.4975) (70,0.4974) (75,0.4974) (80,0.4974) (85,0.4974) (89,0.497)
};

\addplot[blue!50!white, mark=none, thin] coordinates {
(0,0.5) (2,0.4998) (5,0.4994) (10,0.4986) (15,0.4988) (20,0.499) (25,0.4984)
(30,0.4978) (35,0.4971) (40,0.4966) (45,0.4962) (50,0.4957) (55,0.4955) (60,0.4953)
(65,0.4951) (70,0.495) (75,0.495) (80,0.4949) (84,0.4949)
};

\addplot[red!70!black, mark=none] coordinates {
(0,0.5) (1,0.5002) (2,0.5001) (3,0.4999) (4,0.4997) (5,0.4995) (6,0.4993) (7,0.4994)
(8,0.4995) (9,0.4997) (10,0.4999) (11,0.5) (12,0.5002) (13,0.5004) (14,0.5005)
(15,0.5006) (16,0.5007) (17,0.5007) (18,0.5006) (19,0.5006) (20,0.5005) (21,0.5003)
(22,0.5002) (23,0.5) (24,0.4998) (25,0.4997) (26,0.4996) (27,0.4994) (28,0.4993)
(29,0.4992) (30,0.4991) (31,0.499) (32,0.4989) (33,0.4988) (34,0.4987) (35,0.4986)
(36,0.4985) (37,0.4985) (38,0.4984) (39,0.4984) (40,0.4983) (41,0.4983) (42,0.4982)
(43,0.4982) (44,0.4981) (45,0.4981) (46,0.4981) (47,0.498) (48,0.498) (49,0.498)
(50,0.4979) (51,0.4979) (52,0.4979) (53,0.4978) (54,0.4978) (55,0.4978) (56,0.4978)
(57,0.4977) (58,0.4977) (59,0.4977) (60,0.4977) (61,0.4977) (62,0.4976)
};
\addlegendentry{Citeseer ($\bar{k}{=}3.2$)}

\addplot[red!40!white, mark=none, thin] coordinates {
(0,0.5) (2,0.5003) (5,0.4998) (8,0.4993) (10,0.4995) (13,0.5) (16,0.5002)
(19,0.5) (22,0.4996) (25,0.4992) (28,0.499) (31,0.4987) (34,0.4984) (37,0.4981)
(40,0.4979) (43,0.4978) (46,0.4976) (49,0.4975) (52,0.4974) (55,0.4973) (59,0.4971)
};

\addplot[red!40!white, mark=none, thin] coordinates {
(0,0.5) (2,0.4996) (5,0.4992) (8,0.4989) (10,0.4991) (13,0.4994) (16,0.4999)
(19,0.4999) (22,0.4995) (25,0.499) (28,0.4986) (31,0.4982) (34,0.4978) (37,0.4975)
(40,0.4973) (43,0.4972) (46,0.497) (49,0.4969) (52,0.4968) (55,0.4967) (59,0.4966)
};

\end{axis}
\end{tikzpicture}
\caption{Gate trajectories on pairwise datasets (neutral init).
All six runs (3~splits $\times$ 2~datasets; light traces show additional seeds) drift below~$0.5$: gradient descent suppresses the density branch.}
\label{fig:gate-trajectory}
\end{subfigure}%
\hfill
\begin{subfigure}[t]{0.38\textwidth}
\centering
\begin{tikzpicture}
\begin{axis}[
    width=\textwidth,
    height=0.93\textwidth,
    ybar,
    bar width=10pt,
    ylabel={Final $\sigma(g)$},
    symbolic x coords={CE-H, Nat-H, M-O, Cora, Cite, Gene-D},
    xtick=data,
    xticklabel style={font=\footnotesize, rotate=30, anchor=east},
    ymin=0, ymax=0.6,
    ytick={0, 0.1, 0.2, 0.3, 0.4, 0.5},
    yticklabel style={/pgf/number format/fixed, /pgf/number format/precision=1},
    grid=both,
    grid style={gray!15},
    title style={font=\small},
    nodes near coords style={font=\tiny, above, /pgf/number format/fixed, /pgf/number format/precision=2},
    every node near coord/.append style={yshift=2pt},
]

\addplot[dashed, gray, thin, sharp plot, forget plot] coordinates {(CE-H,0.119) (Nat-H,0.119) (M-O,0.119)};
\addplot[dashed, gray, thin, sharp plot, forget plot] coordinates {(Cora,0.5) (Cite,0.5) (Gene-D,0.5)};

\addplot[fill=hmG!70, draw=hmG!90!black, nodes near coords] coordinates {
    (CE-H, 0.112)
    (Nat-H, 0.108)
    (M-O, 0.111)
    (Cora, 0.496)
    (Cite, 0.497)
    (Gene-D, 0.500)
};

\addplot[only marks, mark=-, mark size=4pt, black, thick, forget plot, error bars/.cd, y dir=both, y explicit] coordinates {
    (CE-H, 0.112) +- (0, 0.004)
    (Nat-H, 0.108) +- (0, 0.007)
    (M-O, 0.111) +- (0, 0.005)
    (Cora, 0.496) +- (0, 0.001)
    (Cite, 0.497) +- (0, 0.001)
    (Gene-D, 0.500) +- (0, 0.000)
};

\node[font=\tiny, gray] at (axis cs:Nat-H,0.155) {init $\sigma({-}2)$};
\node[font=\tiny, gray] at (axis cs:Cora,0.535) {init $\sigma(0)$};
\node[font=\tiny, gray] at (axis cs:Gene-D,0.535) {init $\sigma(0)$};

\end{axis}
\end{tikzpicture}
\caption{Final learned gate $\sigma(g)$ across regimes ($\pm 1\sigma$).
Active (left 3): ${\sim}11\%$. Suppressed (Cora/Cite): ${\sim}50\%$. Indifferent (Gene-D): frozen at $0.500$.}
\label{fig:gate-barplot}
\end{subfigure}
\caption{Gate dynamics validate the profiler's regime prediction.
\textbf{(a)}~On pairwise datasets ($\bar{k} < 4$), starting from a neutral initialization $\sigma(0) = 0.5$, the density gate drifts monotonically below $0.5$: gradient descent finds no useful density signal and suppresses the branch.
\textbf{(b)}~Three gate regimes emerge: \emph{active} ($\sigma(g) \approx 0.11$, IWS higher-order witnesses---density contributes), \emph{suppression} ($\sigma(g) \approx 0.497$, Cora/Citeseer---density uninformative), and \emph{indifference} ($\sigma(g) = 0.500$, Gene-Disease---zero gradient signal, node features dominate).
Different initializations reflect the profiler's theory-guided configuration (dashed lines); the gate remains near or below initialization except when density features carry genuine discriminative signal.}
\label{fig:gate-dynamics}
\end{figure}
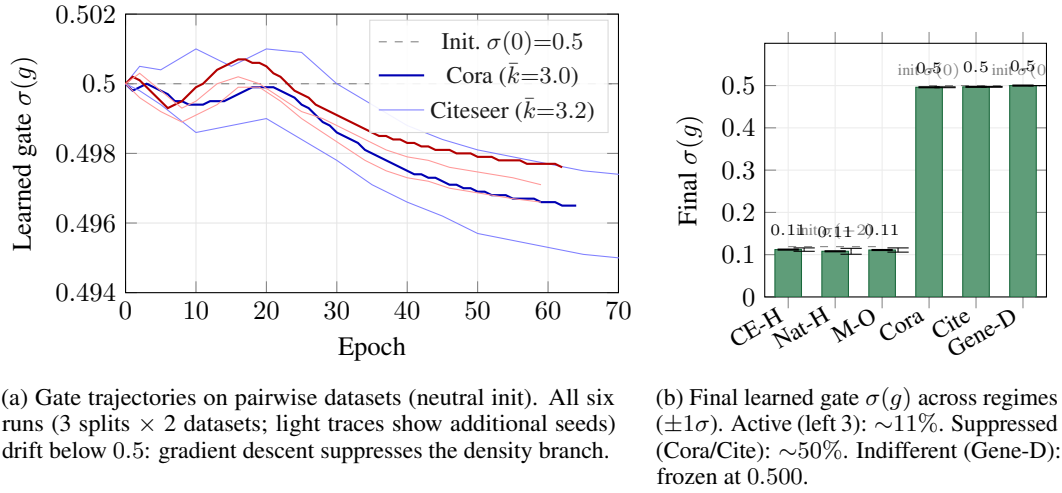

\end{document}